\documentclass[sigconf]{acmart}

\definecolor{Top1}{RGB}{245, 137, 112}
\definecolor{Top2}{RGB}{102, 171, 221}

\def\BibTeX{{\rm B\kern-.05em{\sc i\kern-.025em b}\kern-.08emT\kern-.1667em\lower.7ex\hbox{E}\kern-.125emX}}
    
\copyrightyear{2023}
\acmYear{2023}
\setcopyright{rightsretained}
\acmConference[MM '23]{Proceedings of the 31st ACM International Conference on Multimedia}{October 29-November 3, 2023}{Ottawa, ON, Canada}
\acmBooktitle{Proceedings of the 31st ACM International Conference on Multimedia (MM '23), October 29-November 3, 2023, Ottawa, ON, Canada}
\acmDOI{10.1145/3581783.3611846}
\acmISBN{979-8-4007-0108-5/23/10}

\makeatletter
\gdef\@copyrightpermission{
  \begin{minipage}{0.3\columnwidth}
   \href{https://creativecommons.org/licenses/by/4.0/}{\includegraphics[width=0.90\textwidth]{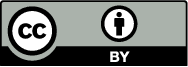}}
  \end{minipage}\hfill
  \begin{minipage}{0.7\columnwidth}
   \href{https://creativecommons.org/licenses/by/4.0/}{This work is licensed under a Creative Commons Attribution International 4.0 License.}
  \end{minipage}
  \vspace{5pt}
}
\makeatother

\acmSubmissionID{671}

\usepackage[linesnumbered,ruled]{algorithm2e}
\usepackage{thmtools}
\usepackage{balance}
\usepackage{thm-restate}
\usepackage{cleveref}
\usepackage{multirow}
\usepackage{color}
\usepackage{tcolorbox}
\tcbuselibrary{listings,skins}
\usepackage{tikz}
\newenvironment{qbox}
	{\begin{tcolorbox}[enhanced jigsaw, drop shadow=black!50!white,colback=white, width=0.95\linewidth, center, left=2pt,right=2pt,top=1pt,bottom=1pt,halign=center]}
	{\end{tcolorbox}}

\newenvironment{qbox-tight}
	{\begin{tcolorbox}[enhanced jigsaw, drop shadow=black!50!white,colback=white, width=\linewidth, center, left=2pt,right=2pt,top=1pt,bottom=1pt,halign=center]}
	{\end{tcolorbox}}
\definecolor{Top1}{RGB}{245, 137, 112}
\begin{document}
\title{When Measures are Unreliable: Imperceptible Adversarial Perturbations toward Top-$k$ Multi-Label Learning}

\author{Yuchen Sun}
\affiliation{%
 \institution{IIP, ICT, CAS}
 \streetaddress{}
 \city{}
 \state{}
 \country{}}
\email{sunyuchen22s@ict.ac.cn}

\author{Qianqian Xu}
\authornotemark[1]
\affiliation{%
 \institution{IIP, ICT, CAS}
 \city{}
 \country{}
}
\email{xuqianqian@ict.ac.cn}

\author{Zitai Wang}
\affiliation{%
 \institution{SKLOIS, IIE, CAS}
 \city{}
 \country{}
}
\affiliation{%
 \institution{SCS, UCAS}
 \city{}
 \country{}
}
\email{wangzitai@iie.ac.cn}

\author{Qingming Huang}
\authornotemark[1]
\affiliation{%
 \institution{SCST, UCAS}
 \city{}
 \country{}
}
\affiliation{%
 \institution{IIP, ICT, CAS}
 \city{}
 \country{}
}
\affiliation{%
 \institution{BDKM, CAS}
 \city{}
 \country{}
}
\affiliation{%
 \institution{Peng Cheng Laboratory}
 \city{}
 \country{}
}
\email{qmhuang@ucas.ac.cn}


\begin{abstract}
With the great success of deep neural networks, adversarial learning has received widespread attention in various studies, ranging from multi-class learning to multi-label learning. However, existing adversarial attacks toward multi-label learning only pursue the traditional visual imperceptibility but ignore the new perceptible problem coming from measures such as Precision@$k$ and mAP@$k$. Specifically, when a well-trained multi-label classifier performs far below the expectation on some samples, the victim can easily realize that this performance degeneration stems from attack, rather than the model itself. Therefore, an ideal multi-labeling adversarial attack should manage to not only deceive visual perception but also evade monitoring of measures. To this end, this paper first proposes the concept of measure imperceptibility. Then, a novel loss function is devised to generate such adversarial perturbations that could achieve both visual and measure imperceptibility. Furthermore, an efficient algorithm, which enjoys a convex objective, is established to optimize this objective. Finally, extensive experiments on large-scale benchmark datasets, such as PASCAL VOC 2012, MS COCO, and NUS WIDE, demonstrate the superiority of our proposed method in attacking the top-$k$ multi-label systems. Our code is available at: \textbf{\url{https://github.com/Yuchen-Sunflower/TKMIA}}.
\end{abstract}



\received{5 May 2023}
\received[revised]{30 June 2023}
\received[accepted]{25 July 2023}

\maketitle
\renewcommand{\thefootnote}{\fnsymbol{footnote}}
\footnotetext[1]{Corresponding authors.}

\begin{figure*}[t]
  \centering
  \includegraphics[width=0.97\linewidth]{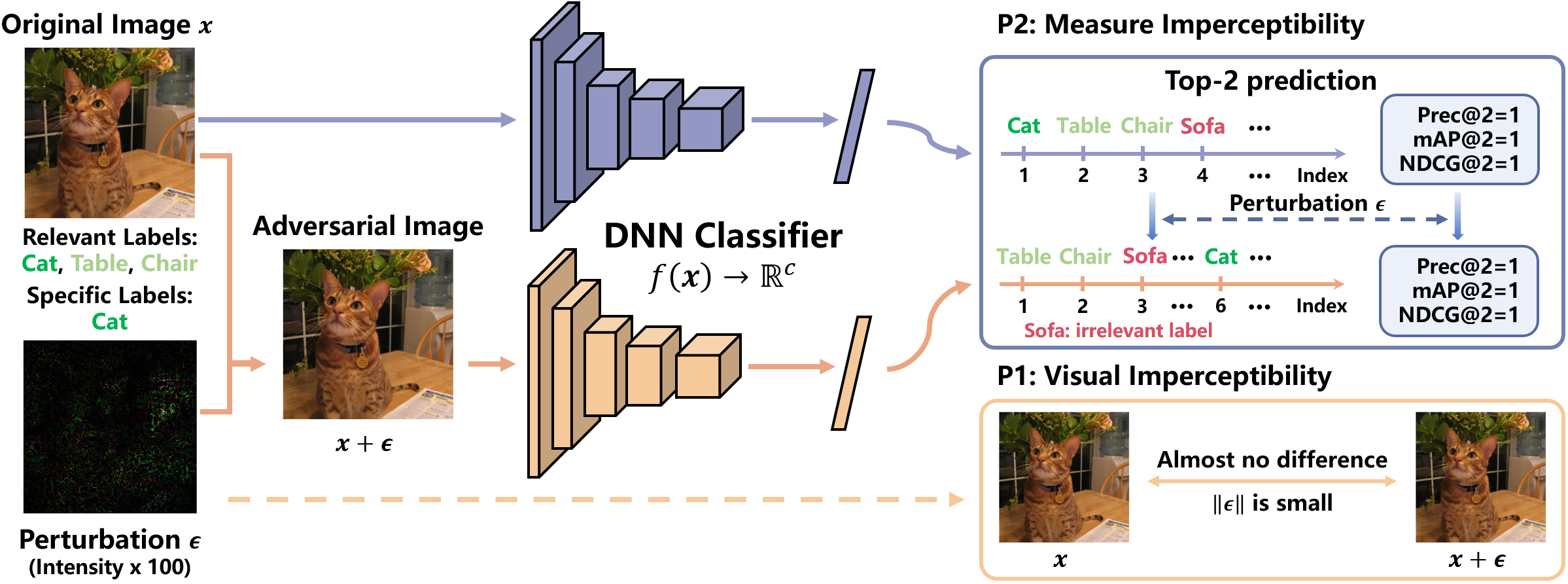}
\caption{The illustration of the proposed imperceptible adversarial perturbation.}
  \vspace{-3mm}
\label{fig:pipeline}
\end{figure*}

\section{Introduction}
\label{sec:introduction}
With the rapid development of Deep Neural Networks (DNNs), machine learning has achieved great success in various tasks, such as image recognition \cite{DBLP:conf/mm/VuLEJN20, DBLP:conf/iccv/ZhaoYZGH021}, text detection \cite{DBLP:conf/aaai/QiaoTCXNPW20, DBLP:conf/aaai/WangLZYBXHW020}, and object tracking \cite{DBLP:conf/cvpr/Wang0BHT19, DBLP:conf/cvpr/MeinhardtKLF22}. However, recent studies have proposed that DNNs are highly vulnerable to the well-crafted perturbations \cite{DBLP:conf/icdm/SongJHH18, DBLP:conf/ijcai/XiaoLZHLS18, DBLP:conf/cvpr/WangDXZCW022}. Due to their intriguing vulnerability, these unguarded models usually suffer from malicious attacks and their performance can be easily affected by slight perturbations hidden behind images. Recently, many pieces of research on adversarial perturbations have found the underlying weakness in practical multimedia applications of DNNs like key text detection \cite{DBLP:conf/mm/XiangLGL022, DBLP:conf/aaai/JinJZS20}, video analysis \cite{DBLP:conf/cvpr/WeiCWJ22, DBLP:conf/ijcai/WangSY21}, and image super-restoration \cite{DBLP:conf/mm/YueLWLL21,
DBLP:conf/iccv/ChoiZKHL19}. The emergence of such problems, which seriously prevent the exploration of various studies and safe applications of DNNs, has raised the attention of many researchers to search for new undetected perturbation technologies, so as to exclude more defense blind spots. 



So far, researchers have developed numerous \textit{multi-class} adversarial perturbation algorithms to protect systems. Most of these effective approaches have been widely adopted in real-world classification scenarios \cite{DBLP:journals/corr/abs-2206-12685, DBLP:conf/nips/SriramananABR20, DBLP:conf/nips/KangSDT21}. In recent years, more attention has shifted to exploring potential threats in \textit{multi-label} learning \cite{DBLP:conf/icdm/SongJHH18, DBLP:journals/pami/MelacciCSDBGR22}. In a multi-label system, common inference strategies can be divided into two categories: \textit{threshold-based} ones and \textit{ranking-based} ones. In the threshold-based approach, the classifier compares the label score with a pre-defined threshold to decide whether each label is relevant to the input sample. Although this strategy is intuitive, the optimal threshold is generally hard to determine. By contrast, the ranking-based approach returns the $k$ labels with the highest scores as the final prediction, whereas the hyper-parameter $k$ is much easier to select according to the scenario. Hence, top-$k$ multi-label learning ($\textup{T}_k$ML) has been widely applied to various multimedia-related tasks, such as information retrieval \cite{DBLP:conf/mm/Bleeker22,
DBLP:conf/mm/BermejoBYMS0020} and recommendation systems \cite{DBLP:conf/mm/LiXJCH20, DBLP:conf/mm/ChenX00CH22}.



Nevertheless, the majority of current studies on multi-label adversarial perturbations still follow the attack principle on multi-class learning, overly valuing the property of \textbf{visual imperceptibility}, but overlooking the new issues that come from the new scenarios. To be specific, for most single-label (multi-class/binary) classifications, when the only relevant category is misclassified, the victim may intuitively attribute the issue to the poor performance of the model, instead of considering whether the attack happens. In contrast, in a multi-label classification scenario, the victim could more likely perceive attacks when all relevant labels are ranked behind the top-$k$ position, as the classification performance of the model is well below the expectation. In view of this, we propose the concept of \textbf{measure imperceptibility}. Empirically, we find that it is quite easy to notice the existence of attack by common multi-label measures ranging from Accuracy, Top-$k$ Accuracy (T$_k$Acc), to ranking-based measures such as Precision at $k$ (P@$k$) \cite{DBLP:conf/ijcai/LiuDLZ16}, mean Average Precision at $k$ (mAP@$k$) \cite{DBLP:conf/cikm/LiGX17}, and Normalized Discounted Cumulative Gain at $k$ (NDCG@$k$) \cite{DBLP:journals/corr/abs-2210-03968}. Hence, a natural question arises:
\begin{qbox}
    \textit{Can we craft an perturbation that satisfies both \textbf{(P1)} visual imperceptibility and \textbf{(P2)} measure imperceptibility?}
\end{qbox}

To answer this question, we introduce a novel algorithm to generate such perturbations for $\textup{T}_k$ML. Specifically, the core challenge lies in constructing an objective function that can disrupt the top-$k$ ranking
results in a well-designed manner. To illustrate this, we provide a simple example in Figure \ref{fig:pipeline}, where the input image $\boldsymbol{x}$ has three relevant labels. Our goal is to generate adversarial perturbations that can prevent the classifier from detecting a specified label such as \texttt{Cat}. In order to meet \textbf{(P1)}, we need to add a slight perturbation $\boldsymbol{\epsilon}$ into the image, which makes it almost no different from the original image. To satisfy \textbf{(P2)}, we can no longer attack all relevant categories. Instead, the other relevant labels like \texttt{Table} and \texttt{Chair} should be ranked higher to compensate for the performance degeneration induced by the attack on the specified label \texttt{Cat}. Then, the top-$k$ measures, such as P@$k$, mAP@$k$, and NDCG@$k$, will not drop significantly while the specified label has been ranked far behind the top-$k$ position. In the extreme case, the classifier can never detect certain important categories and produce degraded classification performance, which may result in serious consequences for multimedia applications, such as multimodal emotion recognition \cite{DBLP:conf/aaai/MittalBCBM20} and text-to-image generation \cite{DBLP:conf/mm/HuangMWW022}.

Inspired by this example, we construct a novel objective that not only guarantees slight perturbations but also pursues well-designed ranking results. When optimizing this objective, both visual imperceptibility and measure imperceptibility can be achieved. Furthermore, we also propose an effective optimization algorithm to efficiently optimize the perturbation. Finally, we validate the effectiveness of our proposed method on three large-scale benchmark datasets (PASCAL VOC 2012, MS COCO 2014, and NUS WIDE). To summarize, the main contributions of this work are three-fold:

\begin{itemize}
    \item To the best of our knowledge, we are the pioneer to introduce the vital concept of measure imperceptibility in adversarial multi-label learning. 
    \item We propose a novel imperceptible attack method toward the $\textup{T}_k$ML problem, which enjoys a convex objective. And an efficient algorithm is established to optimize the proposed objective.
    \item Extensive experiments are conducted on large-scale benchmark datasets including PASCAL VOC 2012, MS COCO 2014, and NUS WIDE. And the empirical results show that the proposed attack method can achieve both visual and measure imperceptibility.
\end{itemize}

\section{Related Work}
\label{sec:related_work}
In this section, we will present a brief introduction to multi-class adversarial perturbations, multi-label adversarial perturbations, and top-$k$ ranking optimization.

\subsection{Multi-class Adversarial Perturbations}
\label{subsec:rel_multi_class}


\textbf{Image-specific Perturbations} is a vital part of adversarial learning, which aims to attack each specific image with a single perturbation. Szegedy \textit{et al.} \cite{DBLP:journals/corr/SzegedyZSBEGF13} first observes the vulnerability of DNNs and gives the corresponding robustness analysis. Then, Goodfellow \textit{et al.} \cite{DBLP:journals/corr/GoodfellowSS14} proposes FGSM and describes the important concept of untargeted attack and targeted attack for image-specific perturbations. Inspired by these works, more perturbation methods are gradually emerging to explore these two directions. Specifically, the untargeted attack is an instance-wise attack that increases the scores of irrelevant categories to replace relevant categories in the top-$k$ set. For example, Deepfool \cite{DBLP:conf/cvpr/Moosavi-Dezfooli16}, as a geometric multi-class untargeted attack, carefully crafts the small perturbation to forge the top-2 class into the correct class. On top of this, $k$Fool \cite{DBLP:conf/wacv/TursynbekPO22} extends the perturbation to the top-$k$ multi-class learning, designed as the corresponding attacks to replace the ground-truth labels in the top-$k$ set with other irrelevant labels. Different from the above attacks, the targeted attack uses specific negative labels to perturb the predictions. FGSM \cite{DBLP:journals/corr/GoodfellowSS14} and I-FGSM \cite{DBLP:conf/iclr/KurakinGB17} conduct back-propagation with the gradient of DNNs to achieve the exploration of minimum efficient perturbation. C\&W \cite{DBLP:conf/sp/Carlini017} presents a typical multi-class adversarial learning framework under the top-1 protocol to implement the targeted attack. On the basis of C\&W, $\textup{CW}^k$ \cite{DBLP:conf/cvpr/ZhangW20} presents an ordered top-$k$ attack by adversarial distillation framework toward top-$k$ multi-class targeted attack. Based on FGSM-based methods, MI-FGSM \cite{DBLP:conf/cvpr/DongLPS0HL18} further introduces the momentum term into an iterative algorithm to avoid coverage into the local optimum. 

Our attack method has a similar effect as the untargeted attack, \textit{i.e.}, it replaces the relevant category in the top-$k$ prediction with other several irrelevant labels. The difference is that our perturbation only focuses on specified relevant categories to promise a slight change in measures.

\textbf{Image-agnostic Perturbations} is another type of emerging attack based on the untargeted attack, which dedicates to generating an overall perturbation to perturb all points. To the best of our knowledge,  Moosavi-Dezfooli \textit{et al.} \cite{DBLP:conf/cvpr/Moosavi-Dezfooli17} first introduces Universal Adversarial Perturbations (UAPs) method, which is independent to the single example. $k$UAP \cite{DBLP:conf/wacv/TursynbekPO22}, as the top-$k$ version of UAPs, focuses on top-$k$ universal adversarial learning in multi-class learning. SV-UAP \cite{DBLP:conf/cvpr/KhrulkovO18} utilizes the singular vectors of the Jacobian matrices of the feature maps to accumulate the minimum perturbation. As interest in UAPs grows recently, a series of related approaches were carried out to further analyze the existing problem in various DNNs-based scenarios \cite{DBLP:conf/cvpr/YuYTK22, DBLP:conf/iccv/0009JLZ0WH19, DBLP:conf/aaai/LiYWYH22, DBLP:conf/cvpr/ZhangBIK20, DBLP:journals/tip/DengK22}.

\subsection{Multi-label Adversarial Perturbations}
\label{subsec:rel_multi_label}

As far as we know, the first study toward multi-label adversarial attack is \cite{DBLP:conf/icdm/SongJHH18}, in which the authors propose a type of targeted attack, extending image-specific perturbations to multi-label learning. In particular, ML-DP and ML-CW are the multi-label extensions of DeepFool and C\&W, respectively. Besides, Melacci \textit{et al.} \cite{DBLP:journals/pami/MelacciCSDBGR22} observe the importance of domain-knowledge constraints for detecting adversarial examples, which analyzes the robustness problem of untargeted attacks in multi-label learning. After realizing the importance of top-k learning in multi-label classification, Hu \textit{et al.} \cite{DBLP:conf/iccv/HuK0L21} first introduces the image-specific perturbations into top-$k$ multi-label learning, where the untargeted/targeted attack attempts to replace all ground-truth labels in the top-$k$ range with arbitrary/specific $k$ negative labels. Also, this work is the first to study the vulnerability of $\textup{T}_k$ML. 

Despite the great work toward multi-class/label adversarial learning, there is no approach focused on the invisibility of perturbations for measures. By careful consideration, the previous works have been more concerned with the efficiency of the perturbation while ensuring that the perturbation is not visually recognizable. However, the occurrence of these attacks can be easily identified when we pay slight attention to the numerical changes in the measure. In this work, we thus propose an imperceptible adversarial perturbation that could easily deceive the common multi-label metrics to fill this study gap.


\vspace{-2mm}
\subsection{Top-$k$ Ranking Optimization}
\label{subsec:rel_topk}

Top-$k$ ranking optimization has a significant amount of applications in various fields, such as binary classification \cite{DBLP:conf/pkdd/FreryHSCH17, DBLP:conf/nips/FanLYH17}, multi-class classification \cite{DBLP:conf/iclr/BerradaZK18, DBLP:conf/kdd/ChangYY17, DBLP:journals/pami/WangXYHCH23}, and multi-label learning \cite{DBLP:conf/nips/HuY0L20}. Due to the discontinuity of individual top-$k$ error, it is computationally hard to directly minimize the top-$k$ loss. Therefore, the practical solution to the top-$k$-related problem usually relies on minimizing a differentiable and convex surrogate loss function \cite{DBLP:conf/icml/YangK20}. In particular, one of the relevant surrogates is the average top-$k$ loss (AT$_k$) \cite{DBLP:conf/nips/FanLYH17}, which generalizes maximum loss and average loss to better fit the optimization in different data distributions, especially for imbalanced problems. On top of this work, Hu \textit{et al.} further propose a summation version of the top-$k$ loss called "Sum of Ranked Range (SoRR)", and demonstrate its effectiveness in top-$k$ multi-label learning \cite{DBLP:conf/nips/HuY0L20}. Both of these surrogates are convex with respect to each individual loss, which means optimizing the objective function only needs simple gradient-based methods. 

Inspired by these discoveries, we introduce the AT$_k$ into our algorithm to further improve the performance of our method. More details about the technology we used to optimize perturbations could be found in Sec.\ref{subsec:topk_optimization}.

\section{Preliminaries}
\label{subsec:preliminary}

In this section, we first give the used notation in the following content and introduce the average of top-$k$ optimization method.

\subsection{Notations}
\label{subsec:notation}

Generally, we assume that the samples are drawn \textit{i.i.d} from the space $\mathcal{Z} = \mathcal{X} \times \mathcal{Y}$, where $\mathcal{X}$ is the input space; $\mathcal{Y} = \{0, 1\}^c$ is the label space with 1 for positive and 0 for negative; $c$ represents the number of labels. In multi-label learning, each input $\boldsymbol{x} \in \mathcal{X}$ is associated with a label vector $\boldsymbol{y} \in \mathcal{Y}$. We use $Y_p = \{ i|y_i = 1 \}$ and $Y_n = \{ j|y_j = 0 \}$ to denote the relevant labels and irrelevant labels of $\boldsymbol{x}$, respectively, where $y_i$ is the $i$-th element of $\boldsymbol{y}$. 

Then, our task is to learn a multi-label prediction function $F(\boldsymbol{x}) = [f_1(\boldsymbol{x}), \dots, f_c(\boldsymbol{x})]: \mathcal{X} \rightarrow \mathbb{R}^c$ to estimate the relevancy score for each class. For each component of $F$, $f_i(\boldsymbol{x}) \in [0, 1]$ denotes the prediction score of $i$-th class. In ranking values, we define $f_{[i]}(\boldsymbol{x})$ as the $i$-th largest element of $F(\boldsymbol{x})$, that is, $f_{[1]}(\boldsymbol{x}) \geq f_{[2]}(\boldsymbol{x}) \geq \dots \geq f_{[c]}(\boldsymbol{x})$. In the following content, we uniformly assume that the ties between any two prediction scores are completely broken, \textit{i.e.}, $f_{[1]}(\boldsymbol{x}) > \cdots > f_{[c]}(\boldsymbol{x})$. 




\subsection{Average of Top-$k$ Optimization}
\label{subsec:topk_optimization}

Top-$k$ optimization, which is a useful tool to construct our objective, has been widely applied to various tasks like information retrieval \cite{DBLP:conf/acl/LinJHW20} and recommendation systems \cite{DBLP:conf/aaai/MaMZWLC21}. For T$_k$ML, it is intuitive to construct an objective function such that the lowest value of the predicted score for the relevant category is higher than the $(k + 1)$-th largest score \cite{DBLP:conf/nips/HuY0L20}: 
\begin{equation*}
  \min_{y \in Y_p} f_{y}(\boldsymbol{x}) > f_{[k + 1]}(\boldsymbol{x}).
\end{equation*}
And this idea has been successfully applied to generate top-$k$ multi-label adversarial perturbation \cite{DBLP:conf/iccv/HuK0L21}, making the perturbed image satisfy that:
\begin{equation*}
    \max_{y \in Y_p}f_y(\boldsymbol{x} + \boldsymbol{\epsilon}) < f_{[k + 1]}(\boldsymbol{x} + \boldsymbol{\epsilon}).
\end{equation*}
In other words, this objective expects the highest score of the relevant labels is not ranked before the top-$k$ position. However, due to the non-differentiable ranking operator, the individual top-$k$ loss is not easy to optimize \cite{DBLP:conf/nips/FanLYH17, DBLP:conf/icml/YangK20}. To address this issue, a series of surrogate loss functions including the sum of ranked range \cite{DBLP:conf/nips/HuY0L20}, average top-$k$ loss \cite{DBLP:conf/nips/FanLYH17}, are proposed. To be specific, for a set of $F(\boldsymbol{x})$, the average top-$k$ loss $\phi_k (F(\boldsymbol{x}))$ is defined as
\begin{equation}
    \phi_k (F(\boldsymbol{x})) = \frac{1}{k} \sum_{i=1}^{k}f_{[i]}(\boldsymbol{x}),
    \label{eq:topk_sum}
\end{equation}
which is proven convex \textit{w.r.t} each component. For this expression, we could directly see that $\phi_k (F(\boldsymbol{x}))$ is the upper bound of $f_{[k]}(\boldsymbol{x})$, \textit{i.e.}, $\phi_k (F(\boldsymbol{x})) \geq f_{[k]}(\boldsymbol{x})$. Hence, this fact provides a convex equivalent form for the original top-$k$ optimization problem.


\section{Measure Imperceptible Attack}
\label{sec:method}

In this section, we define the concept of measure-imperceptible attack and further present the corresponding imperceptible attack algorithm. 

\subsection{Problem Formation}
\label{sec:formation}

This measure-imperceptible attack attempts to attack a few categories while not degenerating the performance of ranking measures. To achieve this goal, we need to
\begin{itemize}
    \item \textbf{G1}: rank the specific relevant categories lower than the $k$-th position;
    \item \textbf{G2}: rank the other relevant categories higher than the $k$-th position as much as possible;
    \item \textbf{G3}: generate an visual imperceptible perturbation $\boldsymbol{\epsilon}$ that is as smaller as possible.
\end{itemize}
According to these sub-goals, we could utilize the existing symbols to form the following optimization problem:
\begin{equation}
    \begin{aligned}
        \min_{\boldsymbol{\epsilon}} \left \| \boldsymbol{\epsilon} \right \|_2 & + \underbrace{f_{[k]}(\boldsymbol{x} + \boldsymbol{\epsilon}) - \min_{y \in Y_p \setminus \mathcal{S}}f_y(\boldsymbol{x} + \boldsymbol{\epsilon})}_{c_2} \\
        s.t. & \underbrace{f_{[k + 1]}(\boldsymbol{x} + \boldsymbol{\epsilon}) \geq \max_{s \in \mathcal{S}} f_{s}(\boldsymbol{x} + \boldsymbol{\epsilon})}_{c_1}, \\
    \end{aligned}
    \label{eq:targeted}
\end{equation}
where $\mathcal{S}$ is the specified label set that includes partial indices of the relevant categories of a sample $\boldsymbol{x}$.

Specifically, the condition $c_1$ guarantees that all specified categories cannot be ranked higher than the top-$k$ position, $c_2$ requires the top-$k$ set to be filled with a certain amount of other relevant labels, and the regularizer limits the size of perturbations, respectively. These three constraints correspond to the above three sub-goals. However, it is challenging to optimize due to the complex constraints of this problem. Consequently, we expect to seek a more easily optimized formulation to solve this problem. 

\subsection{Optimization Relaxation}
\label{sec:relaxation}
To optimize the problem in Eq.(\ref{eq:targeted}), it is necessary to build a feasible optimization framework. To begin with, we could notice that for a well-trained model, most of the specified categories of an image are possibly ranked higher than the $k$-th largest score, that is, 
\begin{equation*}
    \max_{s \in \mathcal{S}}f_s(\boldsymbol{x} + \boldsymbol{\epsilon}) - f_{[k + 1]}(\boldsymbol{x} + \boldsymbol{\epsilon}) \geq 0.
\end{equation*}
Meanwhile, there exist a few ground-truth labels that are not in the top-$k$ set, and they satisfy the condition that: 
\begin{equation*}
    f_{[k]}(\boldsymbol{x} + \boldsymbol{\epsilon}) - \min_{y \in Y_p \setminus \mathcal{S}}f_{y}(\boldsymbol{x} + \boldsymbol{\epsilon}) \geq 0.
\end{equation*}
Thus, the original problem (\ref{eq:targeted}) could be transformed to reduce the loss generated by combining the above two equations and the perturbation norm. Utilizing the Lagrangian equation, we obtain this surrogate loss function:

\begin{equation}
    \begin{aligned}
        \min_{\boldsymbol{\epsilon}} \frac{\alpha}{2} \left \| \boldsymbol{\epsilon} \right \|_2^2 
        & + \left [ \max_{s \in \mathcal{S}}f_s(\boldsymbol{x} + \boldsymbol{\epsilon}) - f_{[k + 1]}(\boldsymbol{x} + \boldsymbol{\epsilon}) \right ]_+ \\
        & + \left [ f_{[k]}(\boldsymbol{x} + \boldsymbol{\epsilon}) - \min_{y \in Y_p \setminus \mathcal{S}}f_y(\boldsymbol{x} + \boldsymbol{\epsilon}) \right ]_+,
    \end{aligned}
    \label{eq:targeted-surrogate}
\end{equation}
where $\alpha$ is the trade-off hyper-parameter. However, the potential gradient sparsity problem \cite{DBLP:conf/icml/YangK20} of Eq.(\ref{eq:targeted-surrogate}) makes the optimization very hard. To solve this problem, we could further leverage Eq.(\ref{eq:topk_sum}) to form an equivalent convex relaxation. And the following lemma also provides us with a solution for the average top-$k$ optimization:

\begin{restatable}[\cite{DBLP:journals/ipl/OgryczakT03}]{lemma}{equivalence}
For $\lambda \in [0, 1], f_i(\boldsymbol{x}) \in [0, 1], i = 1, \dots c$, we have 
\begin{equation*}
    \sum_{i=1}^{k}f_{[i]}(\boldsymbol{x}) = \min_{\lambda \in [0, 1]}\{ k \lambda + \sum_{i=1}^{c} [f_i(\boldsymbol{x}) - \lambda]_+\},
\end{equation*}
where $[a]_+ = \max\{0, a\}$, and $f_{[k]}(\boldsymbol{x})$ is one optimal solution.
\label{lem:equivalence}
\end{restatable}
Thus by adopting Lem.\ref{lem:equivalence} to Eq.(\ref{eq:topk_sum}), a convex optimization framework could be established. We will show this specific relaxation process of our objective below. 

\textbf{Relaxation of $c_1$.} For the sake of explanation, we first make the definitions of $\Delta_i$ and $\Delta_{[i]}$, which measure the difference between the maximum score of all specified labels and the $i$-th/bottom-$i$ score of all labels, respectively, \textit{i.e.},
\begin{equation}
\begin{aligned}
    \Delta_i = & \left [ \max_{s \in \mathcal{S}} f_s(\boldsymbol{x} + \boldsymbol{\epsilon}) - f_{i}(\boldsymbol{x} + \boldsymbol{\epsilon}) \right ]_+, \\
    \Delta_{[i]} = & \left [ \max_{s \in \mathcal{S}} f_s(\boldsymbol{x} + \boldsymbol{\epsilon}) - f_{[c - i + 1]}(\boldsymbol{x} + \boldsymbol{\epsilon}) \right ]_+.
\end{aligned}
\end{equation}
When $i = 1$, $\Delta_{[1]}$ is the largest difference since $f_{[c]}(\boldsymbol{x} + \boldsymbol{\epsilon})$ is the smallest value among $c$ scores. It could be known that as $i$ increases, the value of $f_{[c - i + 1]}(\boldsymbol{x} + \boldsymbol{\epsilon})$ increases consistently and the difference $\Delta_{[i]}$ reduces. Thus, $\Delta_{[i]}$ refers to the $i$-th largest element among $\Delta_i$. Then, according to the fact 
\begin{equation*}
    \frac{1}{c - k} \sum_{i=1}^{c - k}\Delta_{[i]} \geq \Delta_{[c - k]},
\end{equation*}
and Lem.\ref{lem:equivalence}, we have 
\begin{equation*}
    \min_{\lambda_1 \in [0, 1]} \lambda_1 + \frac{1}{c - k} \sum_{i=1}^{c}\left [ \Delta_i - \lambda_1 \right ]_+.
\end{equation*}
Furthermore, we could remove the inner redundant hinge function by the following lemma: 
\begin{restatable}[\cite{DBLP:conf/nips/FanLYH17}]{lemma}{hinge}
For $ \forall a > 0, b > 0$, we have $[[a - x]_+ - b]_+ = [a - x - b]_+$.
\label{lem:hinge}
\end{restatable}
Then, the second term in Eq.(\ref{eq:targeted-surrogate}) could be finally transformed into
\begin{equation*}
    \min_{\lambda_1 \in [0, 1]} \lambda_1 + \frac{1}{c - k} \sum_{i=1}^{c}\left [ \max_{s \in \mathcal{S}} f_s(\boldsymbol{x} + \boldsymbol{\epsilon}) - f_{i}(\boldsymbol{x} + \boldsymbol{\epsilon}) - \lambda_1 \right ]_+.
\end{equation*}

\textbf{Relaxation of $c_2$.} Similarly, we could make the definitions of $\tilde{\Delta}_j$ and $\tilde{\Delta}_{[j]}$ to convert the third term of Eq.(\ref{eq:targeted-surrogate}):
\begin{equation}
\begin{aligned}
    \tilde{\Delta}_j = & \left [ f_j(\boldsymbol{x} + \boldsymbol{\epsilon}) - \min_{y \in Y_p \setminus \mathcal{S}} f_y(\boldsymbol{x} + \boldsymbol{\epsilon}) \right ]_+, \\
    \tilde{\Delta}_{[j]} = & \left [ f_{[j]}(\boldsymbol{x} + \boldsymbol{\epsilon}) - \min_{y \in Y_p \setminus \mathcal{S}} f_y(\boldsymbol{x} + \boldsymbol{\epsilon}) \right ]_+,
\end{aligned}
\end{equation}
where $\tilde{\Delta}_{[j]}$ is the $j$-th largest element among $\tilde{\Delta}_j$. Correspondingly, its final optimization could be written as:
\begin{equation*}
    \min_{\lambda_2 \in [0, 1]} \lambda_2 + \frac{1}{k} \sum_{j=1}^{c}\left [ f_{j}(\boldsymbol{x} + \boldsymbol{\epsilon}) - \min_{y \in Y_p \setminus \mathcal{S}}f_y(\boldsymbol{x} + \boldsymbol{\epsilon}) - \lambda_2 \right ]_+.
\end{equation*}
Combining the above two objectives, we obtain our final objective. The final objective is shown as follows:

\begin{equation}
\begin{aligned}
    \min_{\lambda_1, \lambda_2 \in [0, 1], \boldsymbol{\epsilon}} & \lambda_1 + \lambda_2 + \frac{\alpha}{2} \left \| \boldsymbol{\epsilon} \right \|^2_2\\
    & + \frac{1}{c - k} \sum_{i=1}^{c}\left [ \max_{s \in \mathcal{S}} f_s(\boldsymbol{x} + \boldsymbol{\epsilon}) - f_{i}(\boldsymbol{x} + \boldsymbol{\epsilon}) - \lambda_1 \right ]_+ \\
    & + \frac{1}{k}\sum_{j=1}^{c} \left [ f_{j}(\boldsymbol{x} + \boldsymbol{\epsilon}) - \min_{y \in Y_p \setminus \mathcal{S}}f_y(\boldsymbol{x} + \boldsymbol{\epsilon}) - \lambda_2 \right ]_+.
    \label{eq:targeted-attack-final}
\end{aligned}
\end{equation}


In this way, our original single-element optimization problem is modified into a form of average top-$k$ optimization. This objective does not depend on any sorted result but just needs to calculate the sum result, which greatly reduces the complexity of optimizing perturbations. To solve this optimization problem, we could utilize the iterative gradient descent method \cite{DBLP:conf/iccv/HuK0L21, DBLP:conf/nips/HuY0L20} to directly update $\boldsymbol{\epsilon}$ and $\lambda_i$ with the following \textbf{Algorithm 1}. For this algorithm, we would adopt two different selection schemes to select specified categories:
\begin{itemize}
    \item \textbf{S1: Global Selection}. In this case, we globally specify several categories for all the samples. Note that only the samples relevant to at least one specific category are attacked.
    \item \textbf{S2: Random Selection}. In this case, we randomly specify $m$ ground-truth labels for each example. For any integer $m$, it satisfies that $0 < m \leq k$.
\end{itemize}

\begin{algorithm}[htbp]
\DontPrintSemicolon
  \SetAlgoLined
  \KwIn {example $\boldsymbol{x}$, predictor $F$, learning rate $\eta$, maximum iteration $L$,  specified label set $\mathcal{S}$}
  \KwOut {Optimal perturbation $\boldsymbol{\epsilon}^{l+1}$, adversarial example $\boldsymbol{x}^*$}
  \textbf{Initialization}: $\boldsymbol{\epsilon}^0 = 0$, $\boldsymbol{x}^* = \boldsymbol{x}$, $l = 0$\;
  \While{$l < L$}{
    Calculate the gradients of Eq.(\ref{eq:targeted-attack-final}): $\nabla \lambda^{l}_1$, $\nabla \lambda^{l}_2$ and $\nabla \boldsymbol{\epsilon}^{l}$  \\
    $\lambda^{l+1}_1 = \lambda^{l}_1 - \eta \nabla \lambda^{l}_1, \quad \lambda^{l+1}_2 = \lambda^{l}_2 - \eta \nabla \lambda^{l}_2$ \\
    Update $\boldsymbol{x}^* = \boldsymbol{x}^* + \boldsymbol{\epsilon}^{l+1}$ where $\boldsymbol{\epsilon}^{l+1} = \boldsymbol{\epsilon}^{l} - \eta \nabla \boldsymbol{\epsilon}^{l}$ \\
    \If {\textup{conditions of Eq.(\ref{eq:targeted}) hold}}{
        break\;
}
}
  \Return{$\boldsymbol{x}^*, \boldsymbol{\epsilon}^{l+1}$}
  \caption{Top-$k$ Measure Imperceptible Attack ($\textup{T}_k$MIA)}
\end{algorithm}

\begin{table*}[htbp]
  \centering
  \caption{The average performance with the maximum iteration 300 under different $k$ values and \underline{globally} selected $\mathcal{S}$ on \underline{COCO}, where $\Delta$ refers to the difference between the original value and the perturbed value of corresponding metrics. $\downarrow$ means the smaller the value the better, and $\uparrow$ is the opposite. The best results under each set of parameters are bolded.}
  \renewcommand\arraystretch{0.85}
  \newcommand{\tabincell}[2]{\begin{tabular}{@{}#1@{}}#2\end{tabular}}
  \setlength{\tabcolsep}{3.5mm}{
    \begin{tabular}{c|cc|cccccc}
    \toprule
    Type & $k$   & Methods & $\Delta \textup{T}_k \textup{Acc} \downarrow$ & $\Delta \textup{P}@k \downarrow$ & $\Delta \textup{mAP}@k \downarrow$ & $\Delta \textup{NDCG}@k \downarrow$ & $\Delta l \uparrow$ & $\textup{APer} \downarrow$ \\
    \midrule
    \midrule
    \multirow{12}[4]{*}{Person} & \multirow{4}[1]{*}{3} &ML-CW-U & 0.6400  & 0.4607  & 0.6064  & 0.5409  & 0.9960  & 1.9235  \\
          &       & $k$Fool & 0.4820  & 0.2643  & 0.2909  & 0.2297  & 0.6540  & 3.0603  \\
          &       & T$_{k}$ML-AP-U & 0.6290  & 0.4466  & 0.5860  & 0.5213  & 1.0000  & 1.5846  \\
          &       & $\textup{T}_k$MIA(Ours) & \textbf{0.1880} & \textbf{0.1073} & \textbf{0.1320} & \textbf{0.1081} & \textbf{1.0000} & \textbf{1.2314} \\
\cmidrule{2-9}          & \multirow{4}[1]{*}{5} &ML-CW-U & 0.3840  & 0.5288  & 0.6872  & 0.6459  & 0.9930  & 2.4513  \\
          &       & $k$Fool & 0.2920  & 0.2514  & 0.2989  & 0.2405  & 0.6570  & 12.5931  \\
          &       & T$_{k}$ML-AP-U & 0.3710  & 0.4904  & 0.6573  & 0.6105  & 0.9980  & 2.0422  \\
          &       & $\textup{T}_k$MIA(Ours) & \textbf{0.0670} & \textbf{0.0844} & \textbf{0.1161} & \textbf{0.0909} & \textbf{1.0000} & \textbf{1.6404} \\
\cmidrule{2-9}          & \multirow{4}[1]{*}{10} &ML-CW-U & 0.1098  & 0.5366  & 0.7097  & 0.6716  & 0.9756  & 2.7744  \\
          &       & $k$Fool & 0.0732  & 0.1927  & 0.2501  & 0.1934  & 0.5976  & 4.8946  \\
          &       & T$_{k}$ML-AP-U & 0.1097  & 0.4646  & 0.6699  & 0.6137  & 0.9878  & 2.3951  \\
          &       & $\textup{T}_k$MIA(Ours) & \textbf{0.0000} & \textbf{0.0707} & \textbf{0.0933} & \textbf{0.0627} & \textbf{1.0000} & \textbf{1.4923} \\
    \bottomrule
    \end{tabular}}%
  \label{tab:coco_global_maxiter300}%
\end{table*}%

\begin{table*}[htbp]
  \centering
  \caption{The average performance with the maximum iteration 300 under \underline{globally} selected $\mathcal{S}$ on \underline{NUS}.}
  \renewcommand\arraystretch{0.85}
  \setlength{\tabcolsep}{3.5mm}{
    \begin{tabular}{c|cc|cccccc}
    \toprule
    Type & $k$   & Methods & $\Delta \textup{T}_k \textup{Acc} \downarrow$ & $\Delta \textup{P}@k \downarrow$ & $\Delta \textup{mAP}@k \downarrow$ & $\Delta \textup{NDCG}@k \downarrow$ & $\Delta l \uparrow$ & $\textup{APer} \downarrow$ \\
    \midrule
    \midrule
    \multirow{12}[4]{*}{Buildings} & \multirow{4}[1]{*}{2} &ML-CW-U & 0.4410  & 0.2725  & 0.3215  & 0.2787  & 0.9820  & 1.2455  \\
          &       & $k$Fool & 0.5350  & 0.4210  & 0.4455  & 0.4174  & 1.0030  & 14.2177  \\
          &       & T$_{k}$ML-AP-U & 0.4290  & 0.2650  & 0.3162  & 0.2742  & 0.9710  & 1.0741  \\
          &       & $\textup{T}_k$MIA(Ours) & \textbf{0.2500} & \textbf{0.1630} & \textbf{0.1745} & \textbf{0.1537} & \textbf{1.0040} & \textbf{0.7156} \\
\cmidrule{2-9}          & \multirow{4}[1]{*}{3} &ML-CW-U & 0.5060  & 0.2703  & 0.3482  & 0.2896  & 0.9440  & 1.6943  \\
          &       & $k$Fool & 0.6300  & 0.4890  & 0.5315  & 0.4829  & \textbf{1.0160} & 23.6683  \\
          &       & T$_{k}$ML-AP-U & 0.5200  & 0.2870  & 0.3720  & 0.3116  & 0.9480  & 1.5096  \\
          &       & $\textup{T}_k$MIA(Ours) & \textbf{0.2830} & \textbf{0.1373} & \textbf{0.1483} & \textbf{0.1172} & 1.0120  & \textbf{0.8673} \\
\cmidrule{2-9}          & \multirow{4}[1]{*}{5} &ML-CW-U & 0.3383  & 0.2419  & 0.3391  & 0.2859  & 0.7626  & 2.1203  \\
          &       & $k$Fool & 0.4131  & 0.4441  & 0.5027  & 0.4414  & \textbf{1.0579} & 38.9787  \\
          &       & T$_{k}$ML-AP-U & 0.3530  & 0.2665  & 0.3840  & 0.3283  & 0.7943  & 1.9588  \\
          &       & $\textup{T}_k$MIA(Ours) & \textbf{0.2224} & \textbf{0.1237} & \textbf{0.1315} & \textbf{0.0975} & 1.0318  & \textbf{1.2240} \\
    \bottomrule
    \end{tabular}}%
  \label{tab:nus_global_maxiter300}%
\end{table*}%

\section{Experiments}
\label{sec:exp}

In this section, we present the experimental settings and evaluate the quantitative results of our method against other comparison methods on three large-scale benchmark datasets and two attack schemes.

\subsection{Experimental Settings}
\label{subsec:exp_setting}
\textbf{Datasets.} We carry out our experiments on three well-known benchmark multi-label image annotation datasets containing PASCAL VOC 2012 (VOC) \cite{DBLP:journals/ijcv/EveringhamEGWWZ15}, MS COCO 2014 (COCO) \cite{DBLP:conf/eccv/LinMBHPRDZ14} and NUS WIDE (NUS) \cite{DBLP:conf/civr/ChuaTHLLZ09}. Considering space constraints, we present the information about these datasets in Appendix \ref{appsubsec:datasets}

\textbf{Evaluation Metrics.} For each individual experiment, we would report Multi-label Top-$k$ accuracy (T$_k$Acc) and three popular ranking-based measures for a more comprehensive comparison, including Precision at $k$ (P@$k$), mean Average Precision at $k$ (mAP@$k$) and Normalized Discounted Cumulative Gain at $k$ (NDCG@$k$). 

With a bit of abuse of the notation, we use $\pi$ to denote the indices of permutations of $F(\boldsymbol{x})$. For the instance $\boldsymbol{z} = (\boldsymbol{x}, \boldsymbol{y})$ and model $F$, $\pi_{k}(F, \boldsymbol{z}) = \left \{ [1], \dots, [k] \right \}$ denotes the top-$k$ classes \textit{w.r.t} input $\boldsymbol{x}$. Hence, the Multi-label Top-$k$ Accuracy  \cite{DBLP:conf/nips/HuY0L20} is defined as follows:

\begin{equation}
    \textup{T}_k \textup{Acc}(F, \boldsymbol{z}) = \mathbb{I}_{[Y_p \subseteq \pi_k(F, \boldsymbol{z})]},
\end{equation}
where $\mathbb{I}_{[\cdot]}$ is the indicator function. Then, we make the following definition of Precision at $k$ \cite{DBLP:conf/ijcai/LiuDLZ16}:
\begin{equation}
    \textup{P}@k(f, \boldsymbol{y}) = \frac{1}{k} \sum_{i=1}^{k} y_{[i]}.
\end{equation}
As an extended concept, Average Precision at $k$ \cite{DBLP:conf/icml/WuZ17,DBLP:conf/iccv/RidnikBZNFPZ21} averages the top-$k$ precision performance at different recall values:
\begin{equation}
    \textup{AP}@k(f, \boldsymbol{y}) = \frac{1}{N_k(\boldsymbol{y})} \sum_{i=1}^{k}y_{[i]} \cdot \textup{P}@k,
\end{equation}
where $N_k(\boldsymbol{y}) = \min \{ k, \left |Y_p  \right | \}$, and mean Average Precision at $k$ \cite{DBLP:conf/cikm/LiGX17, DBLP:journals/corr/abs-2209-13262} measures AP performance on each category:
\begin{equation}
    \textup{mAP}@k(f, \boldsymbol{y}) = \frac{1}{c} \sum_{j=1}^{c}\textup{AP}_j@k(f, \boldsymbol{y}),
\end{equation}
where $\textup{AP}_j@k(f, \boldsymbol{y})$ refers to the AP performance on $j$-th category.

\begin{table*}[htbp]
  \centering
  \caption{The average performance with the maximum iteration 300 and $\delta = 2$ under \underline{randomly} selected $\mathcal{S}$ on \underline{COCO}.}
  \renewcommand\arraystretch{0.85}
  \setlength{\tabcolsep}{4mm}{
    \begin{tabular}{c|cc|cccccc}
    \toprule
    $k$   & $\left | \mathcal{S} \right |$ & Methods & $\Delta \textup{T}_k \textup{Acc} \downarrow$ & $\Delta \textup{P}@k \downarrow$ & $\Delta \textup{mAP}@k \downarrow$ & $\Delta \textup{NDCG}@k \downarrow$ & $\Delta l \uparrow$ & $\textup{APer} \downarrow$ \\
    \midrule
    \midrule
    \multirow{12}[4]{*}{5} & \multirow{4}[1]{*}{2} &ML-CW-U & 0.4900  & 0.5232  & 0.6767  & 0.6177  & 1.9970  & 2.0903  \\
          &       & $k$Fool & 0.4640  & 0.3846  & 0.4694  & 0.3898  & 1.6000  & 59.2396  \\
          &       & T$_{k}$ML-AP-U & 0.4890  & 0.5042  & 0.6638  & 0.6026  & 2.0000  & 1.7409  \\
          &       & $\textup{T}_k$MIA(Ours) & \textbf{0.3680} & \textbf{0.2236} & \textbf{0.2641} & \textbf{0.2207} & \textbf{2.0000} & \textbf{1.0922} \\
\cmidrule{2-9}          & \multirow{4}[1]{*}{3} &ML-CW-U & 0.5780  & 0.3416  & 0.5057  & 0.4158  & 2.0100  & 1.9175  \\
          &       & $k$Fool & 0.5060  & 0.2636  & 0.3320  & 0.2499  & 1.8510  & 30.6267  \\
          &       & T$_{k}$ML-AP-U & 0.5740  & 0.3274  & 0.5018  & 0.4117  & 2.0260  & 1.5879  \\
          &       & $\textup{T}_k$MIA(Ours) & \textbf{0.5020} & \textbf{0.2130} & \textbf{0.2731} & \textbf{0.2091} & \textbf{2.0590} & \textbf{1.2568} \\
\cmidrule{2-9}          & \multirow{4}[1]{*}{5} &ML-CW-U & 0.6505  & 0.1689  & 0.2294  & 0.1605  & 2.0097  & 1.5667  \\
          &       & $k$Fool & 0.5097  & 0.1281  & 0.1539  & 0.0988  & \textbf{2.0291} & 5.7119  \\
          &       & T$_{k}$ML-AP-U & 0.6116  & 0.1543  & 0.2127  & 0.1476  & 2.0097  & 1.3096  \\
          &       & $\textup{T}_k$MIA(Ours) & \textbf{0.4563} & \textbf{0.1146} & \textbf{0.1413} & \textbf{0.0940} & 2.0049  & \textbf{1.0605} \\
    \bottomrule
    \end{tabular}}%
  \label{tab:coco_random_maxiter300_pert2}%
\end{table*}%

\begin{table*}[htbp]
  \centering
  \caption{The average performance with the maximum iteration 300 and $\delta = 2$ under \underline{randomly} selected $\mathcal{S}$ on \underline{NUS}.}
  \renewcommand\arraystretch{0.85}
  \newcommand{\tabincell}[2]{\begin{tabular}{@{}#1@{}}#2\end{tabular}}
  \setlength{\tabcolsep}{4mm}{
    \begin{tabular}{c|cc|cccccc}
    \toprule
    $k$   & $\left | \mathcal{S} \right |$ & Methods & $\Delta \textup{T}_k \textup{Acc} \downarrow$ & $\Delta \textup{P}@k \downarrow$ & $\Delta \textup{mAP}@k \downarrow$ & $\Delta \textup{NDCG}@k \downarrow$ & $\Delta l \uparrow$ & $\textup{APer} \downarrow$ \\
    \midrule
    \midrule
    \multirow{12}[4]{*}{5} & \multirow{4}[1]{*}{2} &ML-CW-U & 0.5240  & 0.3970  & 0.5204  & 0.4519  & 1.3400  & 2.6046  \\
          &       & $k$Fool & 0.5980  & 0.5912  & 0.6585  & 0.6010  & \textbf{1.8900} & 121.4293  \\
          &       & T$_{k}$ML-AP-U & 0.5500  & 0.4854  & 0.6164  & 0.5527  & 1.6300  & 2.2604  \\
          &       & $\textup{T}_k$MIA(Ours) & \textbf{0.3990} & \textbf{0.2334} & \textbf{0.2653} & \textbf{0.2001} & 1.6070  & \textbf{1.9674} \\
\cmidrule{2-9}          & \multirow{4}[1]{*}{3} &ML-CW-U & 0.5564  & 0.2877  & 0.3856  & 0.3087  & 1.5063  & 2.4936  \\
          &       & $k$Fool & 0.6118  & 0.3445  & 0.4190  & 0.3283  & \textbf{2.0178} & 38.8708  \\
          &       & T$_{k}$ML-AP-U & 0.5831  & 0.3359  & 0.4682  & 0.3843  & 1.7280  & \textbf{2.1731} \\
          &       & $\textup{T}_k$MIA(Ours) & \textbf{0.4615} & \textbf{0.2228} & \textbf{0.2594} & \textbf{0.1888} & 1.7602  & 2.1984  \\
\cmidrule{2-9}          & \multirow{4}[1]{*}{5} &ML-CW-U & 0.5833  & 0.1583  & 0.2096  & 0.1413  & 1.7083  & 1.9930  \\
          &       & $k$Fool & \textbf{0.5416} & 0.1416  & 0.1700  & 0.1138  & \textbf{2.2083} & 3.1938  \\
          &       & T$_{k}$ML-AP-U & 0.5833  & 0.1583  & 0.1979  & 0.1361  & 1.7916  & \textbf{1.5910} \\
          &       & $\textup{T}_k$MIA(Ours) & 0.5833  & \textbf{0.1416} & \textbf{0.1541} & \textbf{0.0998} & 1.9583  & 1.6672  \\
    \bottomrule
    \end{tabular}}%
  \label{tab:nus_random_maxiter300_pert2}%
\end{table*}%

Normalized Discounted Cumulative Gain at $k$ (NDCG@$k$) \cite{DBLP:journals/corr/abs-2210-03968} is a kind of listwise ranking measure that allocates decreasing weights to the labels from top to bottom positions, so as to precisely evaluate each ranking result:
\begin{equation}
\begin{aligned}
    \textup{DCG@}k(f, \boldsymbol{y}) & = \sum_{i=1}^{k}\frac{y_{[i]}}{\log_2(i + 1)}, \\
    \textup{IDCG@}k(\boldsymbol{y}) & = \sum_{i=1}^{N_k(\boldsymbol{y})}\frac{1}{\log_2(i + 1)}, \\
    \textup{NDCG@}k(f, \boldsymbol{y}) & = \frac{\textup{DCG@}k(f, \boldsymbol{y})}{\textup{IDCG@}k(\boldsymbol{y})}.
\end{aligned}
\end{equation}

To further measure the overall effectiveness of our perturbation algorithm, it is necessary to evaluate the success rate and size of the perturbation. We now define the average attack effectiveness metric $\Delta l$, which calculates the average change number of top-$k$ specified relevant labels over all images:
\begin{equation}
    \Delta l = \frac{1}{n} \sum_{i=1}^{n}(\left | \mathcal{S} \right | - \left | \mathcal{S}' \right |),
\end{equation}
where $\mathcal{S}'$ is the labels in $\mathcal{S}$ that are still ranked in the top-$k$ region after attacking. And average perturbation (APer), which measures the average successful perturbation among all instances, is presented as: 
\begin{equation}
    \textup{APer} = 
    \frac{1}{n} \sum_{i=1}^{n}\left \| \boldsymbol{\epsilon}_i \right \|_2.
\end{equation}


\textbf{Competitors.} We compare our method with the existing top-$k$ multi-label untargeted attack methods, encompassing T$_{k}$ML-AP-U \cite{DBLP:conf/iccv/HuK0L21}, $k$Fool \cite{DBLP:conf/wacv/TursynbekPO22} and ML-CW-U \cite{DBLP:conf/sp/Carlini017}, since their process of excluding the top-$k$ relevant labels is similar to ours. However, all of them are untargeted attacks and not for the specified categories. For the sake of fairness, we uniformly ask these methods to attack specified categories in the top-$k$ set. Then, we observe the norm of successful perturbation and its metric change. Due to length limitations, we have relegated their details to Appendix \ref{appsubsec:competitors}.

Since these competitors hardly achieve measure imperceptibility, it is unfair to require all of them to meet the strict constraints in Eq.(\ref{eq:targeted}). Thus, we uniformly lower the judging standard of a successful perturbation. Note that when the number of excluded specified categories is no less than threshold $\delta$ after a certain iteration, we term it a successful attack. We set $\delta = \left | \mathcal{S} \right |$ for \textbf{S1}, and ask this integer no larger than $\left | \mathcal{S} \right |$ for \textbf{S2}, respectively. In this case, we compare each baseline with our method by observing the value change of measures and average perturbation norm.

\begin{figure*}[t]
  \centering
  \vspace{-0.3cm}
  \includegraphics[width=0.95\linewidth]{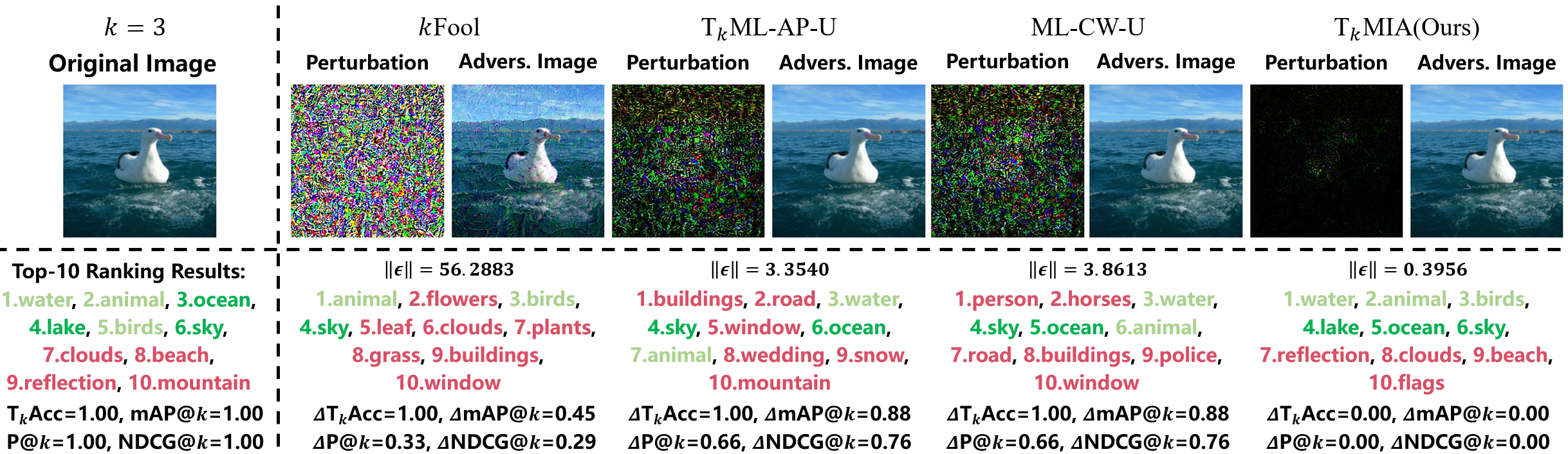}
\caption{The top-3 successful performance comparisons under \textit{global selection scheme} on NUS, where the maximum iteration is 300. All perturbation intensities are magnified by a factor of 100 to enhance contrast and visibility. The specified labels, relevant labels, and irrelevant labels are marked with \textcolor[rgb]{ 0,  .69,  .314}{dark green}, \textcolor[rgb]{ .573,  .816,  .314}{light green}, and \textcolor[rgb]{ .847,  .31,  .4}{red}, respectively.}
\label{fig:nus_global_example}
\end{figure*}

\begin{figure*}[t]
  \centering
  \vspace{-0.3cm}
  \includegraphics[width=0.95\linewidth]{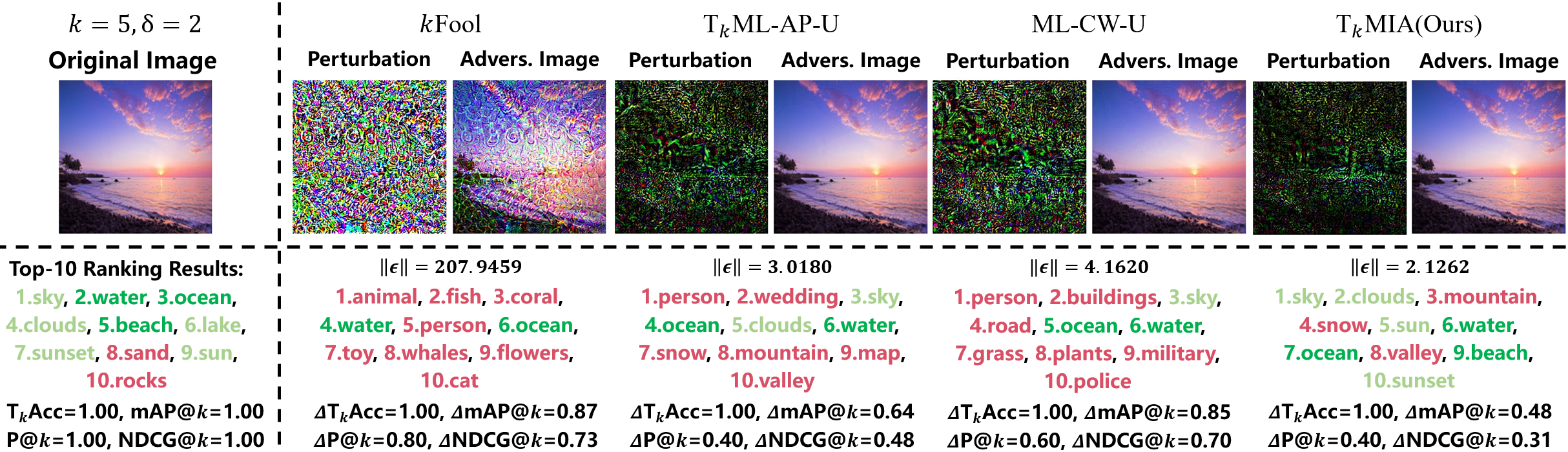}
\caption{The top-5 successful performance comparisons under \textit{random selection scheme} on NUS, where the maximum iteration is 300. All perturbation intensities are magnified by a factor of 100 to enhance contrast and visibility. The specified labels, relevant labels, and irrelevant labels are marked with \textcolor[rgb]{ 0,  .69,  .314}{dark green}, \textcolor[rgb]{ .573,  .816,  .314}{light green}, and \textcolor[rgb]{ .847,  .31,  .4}{red}, respectively.}
\label{fig:nus_random_example}
\vspace{-0.3cm}
\end{figure*}

\textbf{Implementation Details.} Our experiments are implemented by PyTorch \cite{DBLP:conf/nips/PaszkeGMLBCKLGA19}, running on NVIDIA TITAN RTX with CUDA v12.0. We select ResNet-50 \cite{DBLP:conf/cvpr/HeZRS16} as the backbone on VOC, and ResNet-101 \cite{DBLP:conf/cvpr/HeZRS16} on COCO and NUS, respectively. All parameters are initialized with the checkpoint pre-trained on ImageNet, except for the randomly initialized fully connected layer. To satisfy the multi-label classification tasks, we add the sigmoid function to keep the output results ranging from [0, 1]. All pixel intensities of each RGB image range from $\{ 0, 1, \dots, 255\}$. More details about the models we used on each dataset are presented in Appendix \ref{appsubsec:model}.

At the attack period, all images are normalized into the range of $[-1, 1]$. We adopt SGD \cite{DBLP:conf/icml/SutskeverMDH13} as our optimizer, where both learning rate and hyper-parameter $\alpha$ are searched within the range of \{1e-3, 5e-4, 1e-4, 5e-5\}, and momentum adopts 0.9. All of the victim models are well-trained so as to promise the effectiveness of our attack under various specified categories. On top of this, we only sample the instance that satisfies $\left |Y_p  \right | \geq k + \left | \mathcal{S} \right |$ from the validation set as the attack objects. The maximum number of test images under each set of parameters is set as 1000, and the maximum iteration ranges from \{100, 300, 500\}. 

\vspace{-2mm}
\subsection{Results Analysis}
\label{subsec:results}

The average performance of $\textup{T}_k$MIA (1) under the global selection scheme (2) on COCO and NUS are shown in Table \ref{tab:coco_global_maxiter300} and Table \ref{tab:nus_global_maxiter300}, respectively, where the categories included in each type are specifically illustrated in Appendix \ref{appsubsec:classes}. In most cases, $\textup{T}_k$MIA outperforms other comparison methods when $k$ and $\left | \mathcal{S} \right |$ take different values. It could achieve a smaller change in each metric while perturbing more specified categories in $\mathcal{S}$. Although $k$Fool attacks more classes in some cases, it comes at a larger cost to generate the perturbation. Furthermore, Figure \ref{fig:nus_global_example} clearly shows the visual effect of our attack method and other comparison methods in the global selection scheme on NUS. All competitors could manage to remove the specified class from the top-$k$ prediction. However, they raise the scores of several irrelevant categories, making them appear in the top-$k$ predictions, resulting in the performance degradation of measures. Our perturbation considers increasing the scores of other relevant categories in addition to perturbing these specified labels to achieve the effect of confusing the metrics. On the other hand, our method achieves this goal with a smaller perturbation which could hardly be discerned. These results and analysis all demonstrate that T$_k$MIA could achieve both visual and measure imperceptibility in the global selection scheme.

Table \ref{tab:coco_random_maxiter300_pert2} and Table \ref{tab:nus_random_maxiter300_pert2} present the comparative average performance of different methods under different parameters (1) in the random selection scheme (2) on COCO and NUS, respectively. For the measure performance, our perturbation performs better than other baselines. However, an interesting phenomenon is that $k$Fool perturbs more labels while paying a huge effort, and T$_k$ML-AP-U generates a relatively smaller perturbation but only obtains the sub-optimal metric change results. The proposed T$_k$MIA makes a trade-off between visual imperceptibility and measure imperceptibility and achieves a better overall performance on $\Delta l$ and APer. The visualization results of an example under the random selection scheme on NUS are shown in Figure \ref{fig:nus_random_example}. Comparing these value changes in metrics, we can see that T$_k$MIA always promises a smaller change so as to make its perturbation not obvious. Even if there are still some irrelevant labels in the top-$k$ set, our method also tries to prevent them from being in the front position. This means the proposed perturbation is more favorable to maintaining the measure imperceptibility. Other experimental results and analyses on VOC, COCO, and NUS are presented in Appendix \ref{appsec:add_results}.

\section{Conclusion}
\label{sec:conclusion}

In this paper, we propose a method named T$_k$MIA to generate a novel adversarial perturbation toward top-$k$ multi-label classification models, which could satisfy both visual imperceptibility and measure imperceptibility to better evade the monitoring of defenders. To achieve this goal, we form a simple and convex objective with a corresponding gradient iterative algorithm to effectively optimize this perturbation. Finally, a series of empirical studies on different large-scale benchmark datasets and schemes are carried out. Extensive experimental results and analysis validate the effectiveness of the proposed method.


\begin{acks}
This work was supported in part by the National Key R\&D Program of China under Grant 2018AAA0102000, in part by National Natural Science Foundation of China: 62236008, U21B2038, 61931008, 6212200758 and 61976202, in part by the Fundamental Research Funds for the Central Universities,  in part by Youth Innovation Promotion Association CAS, in part by the Strategic Priority Research Program of Chinese Academy of Sciences, Grant No. XDB28000000.
\end{acks}

\vfill\eject
\newpage
\clearpage
\bibliographystyle{ACM-Reference-Format}
\balance
\bibliography{sample-base}

\newpage
\clearpage
\appendix

\section{Proofs}
\label{appsec:proof}

\subsection{Proof of Lemma \ref{lem:equivalence}}
\equivalence*
\begin{proof}
We first define $\boldsymbol{p}$ as a coefficient vector with size $c$, whose component is denoted as $p_i$. It is noted that $\sum_{i=1}^{k}f_{[i]}(\boldsymbol{x})$ is the solution of the following programming problem:
\begin{equation}
    \max_{\boldsymbol{p}} \boldsymbol{p}^{\top}F, \quad s.t.  \boldsymbol{p}^{\top}\textbf{1} = k, \quad 0 \leq p_i \leq 1, \quad i = 1, \dots, c.
    \label{eq:lem1_origin}
\end{equation}
To solve this problem, we could utilize the Lagrangian equation, converting Eq.(\ref{eq:lem1_origin}) to
\begin{equation*}
    L(\boldsymbol{p}, \boldsymbol{u}, \boldsymbol{v}, \lambda) = - \boldsymbol{p}^{\top}F - \boldsymbol{u}^{\top}\boldsymbol{p} + \boldsymbol{v}^{\top}(\boldsymbol{p} - \boldsymbol{1}) + \lambda(\boldsymbol{p}^{\top}\textbf{1} - k),
    \label{eq:lem1_Lag}
\end{equation*}
where non-negative vectors $\boldsymbol{u}$, $\boldsymbol{v}$, and $\lambda \in \mathbb{R}$ are Lagrangian multipliers. Let us take its derivative \textit{w.r.t} $\boldsymbol{p}$ and set it to be $\boldsymbol{0}$, we obtain $\boldsymbol{v} = \boldsymbol{u} - F + \lambda \boldsymbol{1}$. According to this equation, we further get the dual problem of the original problem by substituting it into Eq.(\ref{eq:lem1_origin}):

\begin{equation*}
    \min_{\boldsymbol{u}, \lambda}\boldsymbol{u}^{\top}\boldsymbol{1} + k \lambda, \quad s.t. \boldsymbol{u} \geq \boldsymbol{0}, \boldsymbol{u} - F + \lambda \boldsymbol{1} \geq \boldsymbol{0}.
    \label{eq:lem1_medium}
\end{equation*}
Due to $\boldsymbol{u} \geq F - \lambda \boldsymbol{1} \geq \boldsymbol{0}$, it is easy to know $\lambda \in [0, 1]$, and we thus obtain the final equivalent form:
\begin{equation}
    \sum_{i=1}^{k}f_{[i]}(\boldsymbol{x}) = \min_{\lambda \in [0, 1]}\left \{ k \lambda + \sum_{i=1}^{c} [f_i(\boldsymbol{x}) - \lambda]_+ \right \}.
    \label{eq:lem1_final}
\end{equation}
Through Eq.(\ref{eq:lem1_final}), we observe $f_{[k]}(\boldsymbol{x})$ is always an optimal solution, \textit{i.e.}
\begin{equation}
    f_{[k]}(\boldsymbol{x}) \in \min_{\lambda \in [0, 1]}\left \{ k \lambda + \sum_{i=1}^{c} [f_i(\boldsymbol{x}) - \lambda]_+ \right \}.
    \label{eq:lem1_solution}
\end{equation}
\end{proof}

\subsection{Proof of Lemma \ref{lem:hinge}}
\hinge*
\begin{proof}
Denote $g(x) = [[a - x]_+ - b]_+$. For any $a \geq 0, b \geq 0$, we have $g(x) = 0 = [a - b - x]_+$ when $a \leq x$. In the case of $a > x$, there holds $g(x) = [a - b - x]_+$. Thus $g(x) = [[a - x]_+ - b]_+ = [a - b - x]_+$ holds for any $a \geq 0, b \geq 0$.
\end{proof}

\section{More Experimental Details}
\label{appsec:add_details}

\begin{table}[htbp]
  \centering
  \caption{Summary of well-trained models on different datasets.}
  \renewcommand\arraystretch{1.1}
    \begin{tabular}{c|ccc}
    \toprule
    Dataset & VOC 2012 & COCO 2014 & NUS WIDE \\
    \midrule
    \midrule
    Backbone & ResNet-50 & ResNet-101 & ResNet-101 \\
    Input size & 300   & 448   & 224 \\
    Optimizer & SGD   & Adam  & Adam \\
    Batch size & 32    & 128   & 64 \\
    Learning rate & 1e-4  & 1e-4  & 1e-4 \\
    Weight decay & 1e-4  & 1e-4  & 1e-4 \\
    \midrule
    mAP   & 0.949 & 0.906 & 0.828 \\
    \bottomrule
    \end{tabular}%
  \label{tab:modelinfo}%
\end{table}%

\subsection{Datasets Information}
\label{appsubsec:datasets}
We carry out our experiments on three well-known benchmark multi-label image annotation datasets:
\begin{itemize}
    \item \textbf{PASCAL VOC 2012} (VOC) \cite{DBLP:journals/ijcv/EveringhamEGWWZ15} is a widely used dataset for evaluating the performance of multi-label models toward computer vision tasks. It consists of 10K images with 20 different categories. The dataset is divided into a training set with 5,011 images and a validation set with 4,952 images, respectively. Each image within the training set contains on average 1.43 relevant labels but no more than 6.
    \item \textbf{MS COCO 2014} (COCO) \cite{DBLP:conf/eccv/LinMBHPRDZ14} is another multimedia large-scale dataset that contains 122,218 annotation images with 80 object categories and the corresponding descriptions. These images are divided into two parts, where 82,081 images are for training and 40,137 images are for validation. Each image within the training set contains on average 2.9 labels and a maximum of 13 labels.
    \item \textbf{NUS WIDE} (NUS) \cite{DBLP:conf/civr/ChuaTHLLZ09} is a common dataset consisting of 269,648 real-world web images, where some images are obtained from multimedia data. It covers 81 dedicated categories and each image contains on average 2.4 associated labels. We exclude the images without any labels and randomly divided them into two parts. Specifically, there are 103,990 images in the training set and 69,706 images in the test set. 
\end{itemize}

\subsection{Competitors Introductions}
\label{appsubsec:competitors}
The details about our comparison methods are shown as follows:
\begin{itemize}
    \item T$_{k}$ML-AP-U \cite{DBLP:conf/iccv/HuK0L21}. This is the first top-$k$ untargeted attack method for producing multi-label adversarial perturbation, which is the main baseline to compare with our method. 
    \item $k$Fool \cite{DBLP:conf/wacv/TursynbekPO22}. Inspired by the traditional method DeepFool \cite{DBLP:conf/cvpr/Moosavi-Dezfooli16}, this method is proposed to generate the top-$k$ adversarial perturbation in multi-class classification. To apply to multi-label learning, we consider its prediction score for the only relevant label as the maximum score among all relevant labels like \cite{DBLP:conf/iccv/HuK0L21}.
    \item ML-CW-U. We adopt a similar untargeted vision of C\&W \cite{DBLP:conf/sp/Carlini017} as our comparison method, which adopts the loss function $[\max_{i \notin Y_p}f_i(\boldsymbol{x} + \boldsymbol{\epsilon}) - \min_{j \in Y_p}f_j(\boldsymbol{x} + \boldsymbol{\epsilon})]_+$. Compared to CW$^k$, this method simply extends C\&W to multi-label adversarial learning while not considering the order factor. 
\end{itemize}

\subsection{Model Details}
\label{appsubsec:model}

Table \ref{tab:modelinfo} summarizes the parameters and mAP results of the well-trained models used on each dataset, where the momentum for SGD is empirically set as 0.9. Both ResNet-50 and ResNet-101 are pre-trained on ImageNet \cite{DBLP:conf/cvpr/DengDSLL009}. Utilizing the binary cross-entropy loss, we further fine-tune the model to fit the classification task on different datasets. 

\subsection{Category Description}
\label{appsubsec:classes}

For the experiments under the global selection scheme on VOC, COCO, and NUS, we present the specified categories included in each type in Table \ref{tab:voc_type_info}, Table \ref{tab:coco_type_info} and Table \ref{tab:nus_type_info}, respectively. We select various scales of specified label sets to better show the performance of our method under different parameters.

\begin{table}[htbp]
  \centering
  \caption{Correspondence between types and categories on VOC.}
  \renewcommand\arraystretch{1.1}
  \newcommand{\tabincell}[2]{\begin{tabular}{@{}#1@{}}#2\end{tabular}}
  \setlength{\tabcolsep}{3mm}{
    \begin{tabular}{c|ccccc}
    \toprule
    Type  & \multicolumn{5}{c}{Categories} \\
    \midrule
    \midrule
    Vehicles & \multicolumn{5}{c}{bird, cat, cow, dog, horse, sheep.} \\ 
    Animals & \multicolumn{5}{c}{\tabincell{c}{aeroplane, bicycle, boat, bus, \\ car, motorbike, train.}} \\ 
    Household & \multicolumn{5}{c}{\tabincell{c}{bottle, chair, dining table, \\ potted plant, sofa, TV monitor.}} \\ 
    Person & \multicolumn{5}{c}{person.} \\
    \bottomrule
    \end{tabular}}%
  \label{tab:voc_type_info}%
\end{table}%

\begin{table}[htbp]
  \centering
  \caption{Correspondence between types and categories on COCO.}
  \renewcommand\arraystretch{1.1}
  \newcommand{\tabincell}[2]{\begin{tabular}{@{}#1@{}}#2\end{tabular}}
  \setlength{\tabcolsep}{3mm}{
    \begin{tabular}{c|ccccc}
    \toprule
    Type  & \multicolumn{5}{c}{Categories} \\
    \midrule
    \midrule
    Person & \multicolumn{5}{c}{person.} \\ 
    \tabincell{c}{Daily \\ necessities} & \multicolumn{5}{c}{\tabincell{c}{backpack, umbrella, handbag, tie, \\ suitcase, book, clock, vase, \\ scissors, teddy bear, toothbrush.}} \\
    Furniture & \multicolumn{5}{c}{\tabincell{c}{chair, couch, potted plant, \\ bed, dining table.}} \\
    \tabincell{c}{Electrical \\ equipment} & \multicolumn{5}{c}{\tabincell{c}{TV, laptop, mouse, remote, \\ keyboard, cell phone, microwave, \\ oven, toaster, refrigerator.}} \\
    \bottomrule
    \end{tabular}}%
  \label{tab:coco_type_info}%
\end{table}%

\begin{table}[htbp]
  \centering
  \caption{Correspondence between types and categories on NUS.}
  \renewcommand\arraystretch{1.1}
  \newcommand{\tabincell}[2]{\begin{tabular}{@{}#1@{}}#2\end{tabular}}
  \setlength{\tabcolsep}{2mm}{
    \begin{tabular}{c|ccccc}
    \toprule
    Type & \multicolumn{5}{c}{Categories} \\
    \midrule
    \midrule
    Animals & \multicolumn{5}{c}{\tabincell{c}{elk, bear, cat, zebra, fish, \\ whales, dog, cow, horse, fox, \\ birds, tiger, animal.}} \\ 
    Buildings & \multicolumn{5}{c}{\tabincell{c}{bridge, statue, tower, temple, \\ buildings, house, castle.}} \\ 
    Landscape & \multicolumn{5}{c}{\tabincell{c}{waterfall, rainbow, lake, sunset, \\ moon, ocean, beach, town, \\ glacier, sun, clouds, valley, \\ cityscape, frost, sky, harbor, garden, \\ earthquake, surf, mountain, nighttime.}} \\ 
    Traffic & \multicolumn{5}{c}{\tabincell{c}{plane, airport, train, boats, \\ cars, road, vehicle, street, railroad.}} \\
    \bottomrule
    \end{tabular}}%
  \label{tab:nus_type_info}%
\end{table}%

\subsection{Number of Samples under Different Settings}

As the selected samples have to satisfy $\left |Y_p  \right | \geq k + \left | \mathcal{S} \right |$, the number of samples under various settings is different. Thus, we present the specific numbers of samples and their average specified labels under different settings in Table \ref{tab:number_global_voc}-\ref{tab:number_random}. In particular, Table \ref{tab:number_global_voc}, Table \ref{tab:number_global_coco}, and Table \ref{tab:number_global_nus} show the correspondence under the global selection scheme on each dataset, respectively. And Table \ref{tab:number_random} integrates the number of samples on different datasets under the random selection scheme. Please note that we only select up to 1000 samples as our test images for each experiment.

\begin{table}[htbp]
  \centering
  \caption{Number of samples under different types and $k$ values on VOC.}
  \renewcommand\arraystretch{0.9}
  \newcommand{\tabincell}[2]{\begin{tabular}{@{}#1@{}}#2\end{tabular}}
  \setlength{\tabcolsep}{5mm}{
    \begin{tabular}{c|ccc}
    \toprule
    Type & $k$   & $\left | \mathcal{S} \right |$ & $n$ \\
    \midrule
    \midrule
    \multirow{3}[0]{*}{Vehicles} & 1     & 1.0000  & 365 \\
          & 2     & 1.0384  & 26 \\
          & 3     & 1.0000  & 6 \\
    \midrule
    \multirow{3}[0]{*}{Animals} & 1     & 1.0000  & 213 \\
          & 2     & 1.0000  & 45 \\
          & 3     & 1.0000  & 7 \\
    \midrule
    \multirow{3}[0]{*}{Household} & 1     & 1.0000  & 120 \\
          & 2     & 1.0000  & 28 \\
          & 3     & 1.0000  & 1 \\
    \midrule
    \multirow{3}[0]{*}{Person} & 1     & 1.0000  & 888 \\
          & 2     & 1.0000  & 319 \\
          & 3     & 1.0000  & 59 \\
    \bottomrule
    \end{tabular}%
  \label{tab:number_global_voc}}%
\end{table}%

\begin{table}[htbp]
  \centering
  \caption{Number of samples under different types and $k$ values on COCO.}
  \renewcommand\arraystretch{0.9}
  \newcommand{\tabincell}[2]{\begin{tabular}{@{}#1@{}}#2\end{tabular}}
  \setlength{\tabcolsep}{5mm}{
    \begin{tabular}{c|ccc}
    \toprule
    Type & $k$   & $\left | \mathcal{S} \right |$ & $n$ \\
    \midrule
    \midrule
    \multirow{3}[0]{*}{Person} & 3     & 1.0000  & 1000 \\
          & 5     & 1.0000  & 1000 \\
          & 10    & 1.0000  & 82 \\
    \midrule
    \multirow{3}[0]{*}{\tabincell{c}{Daily \\ necessities}} & 3     & 1.0700  & 1000 \\
          & 5     & 1.1089  & 790 \\
          & 10    & 1.1666  & 30 \\
    \midrule
    \multirow{3}[0]{*}{Furniture} & 3     & 1.1519  & 1000 \\
          & 5     & 1.3539  & 1000 \\
          & 10    & 1.7555  & 45 \\
    \midrule
    \multirow{3}[0]{*}{\tabincell{c}{Electrical \\ equipment}} & 3     & 1.2130  & 1000 \\
          & 5     & 1.3091  & 702 \\
          & 10    & 1.3478  & 23 \\
    \bottomrule
    \end{tabular}}%
  \label{tab:number_global_coco}%
\end{table}%

\begin{table}[htbp]
  \centering
  \caption{Number of samples under different types and $k$ values on NUS.}
  \renewcommand\arraystretch{0.9}
  \newcommand{\tabincell}[2]{\begin{tabular}{@{}#1@{}}#2\end{tabular}}
  \setlength{\tabcolsep}{5mm}{
    \begin{tabular}{c|ccc}
    \toprule
    Type & $k$   & $\left | \mathcal{S} \right |$ & $n$ \\
    \midrule
    \midrule
    \multirow{3}[0]{*}{Animals} & 2     & 1.2710  & 1000 \\
          & 3     & 1.2203  & 590 \\
          & 5     & 1.1818  & 66 \\
    \midrule
    \multirow{3}[0]{*}{Buildings} & 2     & 1.0080  & 1000 \\
          & 3     & 1.0250  & 1000 \\
          & 5     & 1.0636  & 535 \\
    \midrule
    \multirow{3}[0]{*}{Landscape} & 2     & 1.2940  & 1000 \\
          & 3     & 1.4340  & 1000 \\
          & 5     & 1.7027  & 259 \\
    \midrule
    \multirow{3}[0]{*}{Traffic} & 2     & 1.1150  & 1000 \\
          & 3     & 1.0849  & 860 \\
          & 5     & 1.0588  & 255 \\
    \bottomrule
    \end{tabular}%
  \label{tab:number_global_nus}}%
\end{table}%

\begin{table}[htbp]
  \centering
  \caption{Number of samples under random selection scheme.}
  \renewcommand\arraystretch{0.9}
  \newcommand{\tabincell}[2]{\begin{tabular}{@{}#1@{}}#2\end{tabular}}
  \setlength{\tabcolsep}{6mm}{
    \begin{tabular}{c|cc|c}
    \toprule
    Dataset & $k$   & $\left | \mathcal{S} \right |$ & $n$ \\
    \midrule
    \midrule
    \multirow{2}[0]{*}{VOC} & 2     & 2     & 69 \\
          & 3     & 3     & 11 \\
    \midrule
    \multirow{8}[0]{*}{COCO} & \multirow{2}[0]{*}{3} & 2     & 1000 \\
          &       & 3     & 1000 \\
          & \multirow{3}[0]{*}{5} & 2     & 1000 \\
          &       & 3     & 1000 \\
          &       & 5     & 206 \\
          & \multirow{3}[0]{*}{10} & 2     & 63 \\
          &       & 3     & 21 \\
          &       & 5     & 2 \\
    \midrule
    \multirow{6}[0]{*}{NUS} & 2     & 2     & 1000 \\
          & \multirow{2}[0]{*}{3} & 2     & 1000 \\
          &       & 3     & 1000 \\
          & \multirow{3}[0]{*}{5} & 2     & 1000 \\
          &       & 3     & 559 \\
          &       & 5     & 24 \\
    \bottomrule
    \end{tabular}}%
  \label{tab:number_random}%
\end{table}%

\section{Additional Results and Analysis}
\label{appsec:add_results}

\begin{table*}[htbp]
  \centering
  \caption{The rest results of competitors and our method with the maximum iteration 300 under different $k$ values and \underline{globally} selected $\mathcal{S}$ on \underline{COCO}, where $\Delta$ refers to the difference between the original value and the perturbed value of corresponding metrics. $\downarrow$ means the smaller the value the better, and $\uparrow$ is the opposite. The best results under each set of parameters are bolded.}
  \renewcommand\arraystretch{0.9}
  \newcommand{\tabincell}[2]{\begin{tabular}{@{}#1@{}}#2\end{tabular}}
  \setlength{\tabcolsep}{3.5mm}{
    \begin{tabular}{c|cc|cccccc}
    \toprule
    Type & $k$   & Methods & $\Delta \textup{T}_k \textup{Acc} \downarrow$ & $\Delta \textup{P}@k \downarrow$ & $\Delta \textup{mAP}@k \downarrow$ & $\Delta \textup{NDCG}@k \downarrow$ & $\Delta l \uparrow$ & $\textup{APer} \downarrow$ \\
    \midrule
    \midrule
    \multirow{12}[4]{*}{\tabincell{c}{Daily \\ necessities}} & \multirow{4}[1]{*}{3} &ML-CW-U & 0.6040  & 0.3200  & 0.4015  & 0.3303  & 1.0690  & 1.4187  \\
          &       & $k$Fool & 0.6110  & 0.3613  & 0.4047  & 0.3350  & 0.8170  & 13.5042  \\
          &       & T$_{k}$ML-AP-U & 0.6060  & 0.3213  & 0.4112  & 0.3401  & 1.0700  & 1.1906  \\
          &       & $\textup{T}_k$MIA(Ours) & \textbf{0.2660} & \textbf{0.1163} & \textbf{0.1280} & \textbf{0.0958} & \textbf{1.0700} & \textbf{0.5063} \\
\cmidrule{2-9}          & \multirow{4}[1]{*}{5} &ML-CW-U & 0.4474  & 0.3534  & 0.4857  & 0.4168  & 1.1077  & 1.6318  \\
          &       & $k$Fool & 0.4246  & 0.3194  & 0.3845  & 0.3119  & 1.0139  & 18.3515  \\
          &       & T$_{k}$ML-AP-U & 0.4397  & 0.3333  & 0.4672  & 0.4017  & 1.1077  & 1.3634  \\
          &       & $\textup{T}_k$MIA(Ours) & \textbf{0.1825} & \textbf{0.0884} & \textbf{0.1022} & \textbf{0.0729} & \textbf{1.1089} & \textbf{0.5418} \\
\cmidrule{2-9}          & \multirow{4}[1]{*}{10} &ML-CW-U & 0.0000  & 0.3367  & 0.4756  & 0.4389  & 1.1666  & 1.9130  \\
          &       & $k$Fool & 0.0333  & 0.2433  & 0.3237  & 0.2697  & 0.9333  & 23.1311  \\
          &       & T$_{k}$ML-AP-U & 0.0000  & 0.3533  & 0.4905  & 0.4571  & 1.1666  & 1.6887  \\
          &       & $\textup{T}_k$MIA(Ours) & \textbf{0.0000} & \textbf{0.0433} & \textbf{0.0605} & \textbf{0.0384} & \textbf{1.1666} & \textbf{0.5724} \\
    \midrule
    \multirow{12}[4]{*}{Furniture} & \multirow{4}[1]{*}{3} &ML-CW-U & 0.6270  & 0.3497  & 0.4509  & 0.3806  & 1.1510  & 1.2603  \\
          &       & $k$Fool & 0.5960  & 0.3467  & 0.3972  & 0.3300  & 0.6430  & 12.7266  \\
          &       & T$_{k}$ML-AP-U & 0.5990  & 0.3383  & 0.4437  & 0.3755  & 1.1519  & 1.0673  \\
          &       & $\textup{T}_k$MIA(Ours) & \textbf{0.3210} & \textbf{0.1446} & \textbf{0.1605} & \textbf{0.1221} & \textbf{1.1519} & \textbf{0.4519} \\
\cmidrule{2-9}          & \multirow{4}[1]{*}{5} &ML-CW-U & 0.4970  & 0.3970  & 0.5421  & 0.4710  & 1.3540  & 1.7799  \\
          &       & $k$Fool & 0.4800  & 0.3800  & 0.4684  & 0.3920  & 0.8240  & 53.0367  \\
          &       & T$_{k}$ML-AP-U & 0.4860  & 0.3964  & 0.5515  & 0.4821  & 1.3539  & 1.5302  \\
          &       & $\textup{T}_k$MIA(Ours) & \textbf{0.2490} & \textbf{0.1286} & \textbf{0.1513} & \textbf{0.1105} & \textbf{1.3539} & \textbf{0.7216} \\
\cmidrule{2-9}          & \multirow{4}[1]{*}{10} &ML-CW-U & 0.1333  & 0.3933  & 0.5712  & 0.5152  & 1.7333  & 2.1470  \\
          &       & $k$Fool & 0.1333  & 0.2933  & 0.4121  & 0.3296  & 0.7333  & 97.2627  \\
          &       & T$_{k}$ML-AP-U & 0.1333  & 0.4333  & 0.6065  & 0.5514  & 1.7333  & 1.9629  \\
          &       & $\textup{T}_k$MIA(Ours) & \textbf{0.0444} & \textbf{0.1133} & \textbf{0.1291} & \textbf{0.0891} & \textbf{1.7555} & \textbf{0.9502} \\
    \midrule
    \multirow{12}[4]{*}{\tabincell{c}{Electrical \\ equipment}} & \multirow{4}[1]{*}{3} &ML-CW-U & 0.5810  & 0.3323  & 0.4159  & 0.3508  & 1.2130  & 0.9664  \\
          &       & $k$Fool & 0.5730  & 0.3267  & 0.3798  & 0.3149  & 0.9040  & 12.5737  \\
          &       & T$_{k}$ML-AP-U & 0.5580  & 0.3126  & 0.4018  & 0.3372  & 1.2130  & 0.7998  \\
          &       & $\textup{T}_k$MIA(Ours) & \textbf{0.3330} & \textbf{0.1600} & \textbf{0.1758} & \textbf{0.1360} & \textbf{1.2130} & \textbf{0.4253} \\
\cmidrule{2-9}          & \multirow{4}[1]{*}{5} &ML-CW-U & 0.3647  & 0.2786  & 0.4063  & 0.3437  & 1.3077  & 1.2748  \\
          &       & $k$Fool & 0.3818  & 0.2823  & 0.3612  & 0.2886  & 0.9829  & 14.8029  \\
          &       & T$_{k}$ML-AP-U & 0.3589  & 0.2695  & 0.3981  & 0.3368  & 1.3091  & 1.0406  \\
          &       & $\textup{T}_k$MIA(Ours) & \textbf{0.2136} & \textbf{0.1128} & \textbf{0.1360} & \textbf{0.0985} & \textbf{1.3091} & \textbf{0.4702} \\
\cmidrule{2-9}          & \multirow{4}[1]{*}{10} &ML-CW-U & 0.0000  & 0.2652  & 0.4097  & 0.3546  & 1.3478  & 1.6874  \\
          &       & $k$Fool & 0.0000  & 0.1826  & 0.2591  & 0.1874  & 1.0435  & 8.7268  \\
          &       & T$_{k}$ML-AP-U & 0.0000  & 0.2086  & 0.3510  & 0.2944  & 1.3478  & 1.3869  \\
          &       & $\textup{T}_k$MIA(Ours) & \textbf{-0.0434} & \textbf{0.0608} & \textbf{0.0914} & \textbf{0.0679} & \textbf{1.3478} & \textbf{0.5378} \\
    \bottomrule
    \end{tabular}}%
  \label{tab:coco_global_rest_maxiter300}%
\end{table*}%

\subsection{Overall Performance under Scheme 1}

The overall performance of $\textup{T}_k$MIA and other methods under the global selection scheme on COCO, NUS, and VOC are shown in Table \ref{tab:coco_global_rest_maxiter300}, Table \ref{tab:nus_global_rest_maxiter300}, and Table \ref{tab:voc_global_maxiter300}, respectively. 

We can clearly see that our method completely outperforms the other comparison methods. Specifically, our method could generate smaller perturbations to achieve a slight change in metrics. Meanwhile, we obtain a larger $\Delta l$, which means that our approach successfully perturbs more specified categories. On one hand, we could observe that all methods generate larger perturbations with increasing the $k$ value while the increase in our average perturbation norm is smaller than others. For example, on the Traffic type of NUS, when $k$ changes from 2 to 5, the APer of ML-CW-U, $k$Fool, and T$_k$ML-AP-U increase 0.7295, 5.4391, and 0.6772, respectively. Our T$_k$MIA only produces 0.0934 increasing. For some types (Vehicles, Animals, Household types of VOC), its APer even produces a decreasing trend. On the other hand, most changes in measures for all methods get smaller as the $k$ value increases. But when there are large $k$ and $\left | \mathcal{S} \right |$, T$_k$ML-AP-U yields the opposite trend, such as on Daily necessities and Furniture types of COCO. This comparison indicates that our method is more robust in complex conditions. Through these two aspects, we verify the outstanding performance of our method.

\begin{table*}[htbp]
  \centering
  \caption{The rest results of competitors and our method with the maximum iteration 300 under different $k$ values and \underline{globally} selected $\mathcal{S}$ on \underline{NUS}, where $\Delta$ refers to the difference between the original value and the perturbed value of corresponding metrics. $\downarrow$ means the smaller the value the better, and $\uparrow$ is the opposite. The best results under each set of parameters are bolded.}
  \renewcommand\arraystretch{0.9}
  \setlength{\tabcolsep}{3.5mm}{
    \begin{tabular}{c|cc|cccccc}
    \toprule
    Type & $k$   & Methods & $\Delta \textup{T}_k \textup{Acc} \downarrow$ & $\Delta \textup{P}@k \downarrow$ & $\Delta \textup{mAP}@k \downarrow$ & $\Delta \textup{NDCG}@k \downarrow$ & $\Delta l \uparrow$ & $\textup{APer} \downarrow$ \\
    \midrule
    \midrule
    \multirow{12}[4]{*}{Animals} & \multirow{4}[1]{*}{2} &ML-CW-U & 0.4240  & 0.2545  & 0.3043  & 0.2612  & 1.2600  & 1.2388  \\
          &       & $k$Fool & 0.4080  & 0.2465  & 0.2755  & 0.2362  & 1.1290  & 5.1962  \\
          &       & T$_{k}$ML-AP-U & 0.4290  & 0.2585  & 0.3135  & 0.2697  & 1.2649  & 1.1038  \\
          &       & $\textup{T}_k$MIA(Ours) & \textbf{0.2410} & \textbf{0.1485} & \textbf{0.1715} & \textbf{0.1484} & \textbf{1.2710} & \textbf{0.9219} \\
\cmidrule{2-9}          & \multirow{4}[1]{*}{3} &ML-CW-U & 0.4288  & 0.2373  & 0.3096  & 0.2568  & 1.1712  & 1.4019  \\
          &       & $k$Fool & 0.4136  & 0.2508  & 0.2702  & 0.2222  & 1.0339  & 5.8748  \\
          &       & T$_{k}$ML-AP-U & 0.4470  & 0.2485  & 0.3282  & 0.2741  & 1.1847  & 1.2732  \\
          &       & $\textup{T}_k$MIA(Ours) & \textbf{0.1966} & \textbf{0.1040} & \textbf{0.1195} & \textbf{0.0961} & \textbf{1.2203} & \textbf{0.8284} \\
\cmidrule{2-9}          & \multirow{4}[1]{*}{5} &ML-CW-U & 0.3788  & 0.2000  & 0.2684  & 0.2184  & 0.9697  & 1.5046  \\
          &       & $k$Fool & 0.2879  & 0.2424  & 0.2702  & 0.2236  & 0.9848  & 10.4123  \\
          &       & T$_{k}$ML-AP-U & 0.3484  & 0.2121  & 0.3098  & 0.2588  & 0.9545  & 1.4826  \\
          &       & $\textup{T}_k$MIA(Ours) & \textbf{0.1364} & \textbf{0.0424} & \textbf{0.0498} & \textbf{0.0367} & \textbf{1.1818} & \textbf{0.7242} \\
    \midrule
    \multirow{12}[4]{*}{Landscape} & \multirow{4}[1]{*}{2} &ML-CW-U & 0.6160  & 0.4250  & 0.5113  & 0.4598  & 1.2150  & 1.8080  \\
          &       & $k$Fool & 0.6770  & 0.4985  & 0.5273  & 0.4841  & 1.2160  & 36.0644  \\
          &       & T$_{k}$ML-AP-U & 0.6310  & 0.4405  & 0.5313  & 0.4795  & 1.2569  & 1.4559  \\
          &       & $\textup{T}_k$MIA(Ours) & \textbf{0.1300} & \textbf{0.0910} & \textbf{0.0993} & \textbf{0.0896} & \textbf{1.2940} & \textbf{1.0430} \\
\cmidrule{2-9}          & \multirow{4}[1]{*}{3} &ML-CW-U & 0.5660  & 0.3900  & 0.5130  & 0.4537  & 1.2350  & 2.2045  \\
          &       & $k$Fool & 0.6200  & 0.4800  & 0.5141  & 0.4694  & 1.3130  & 70.4729  \\
          &       & T$_{k}$ML-AP-U & 0.5770  & 0.4016  & 0.5337  & 0.4734  & 1.3179  & 1.8304  \\
          &       & $\textup{T}_k$MIA(Ours) & \textbf{0.0770} & \textbf{0.0520} & \textbf{0.0630} & \textbf{0.0537} & \textbf{1.4260} & \textbf{1.3236} \\
\cmidrule{2-9}          & \multirow{4}[1]{*}{5} &ML-CW-U & 0.3205  & 0.2873  & 0.4037  & 0.3605  & 1.0811  & 2.4561  \\
          &       & $k$Fool & 0.4363  & 0.4865  & 0.5635  & 0.5044  & 1.4131  & 159.2512  \\
          &       & T$_{k}$ML-AP-U & 0.3359  & 0.3220  & 0.4665  & 0.4200  & 1.2548  & 2.1796  \\
          &       & $\textup{T}_k$MIA(Ours) & \textbf{-0.0502} & \textbf{0.0232} & \textbf{0.0385} & \textbf{0.0337} & \textbf{1.6525} & \textbf{1.7653} \\
    \midrule
    \multirow{12}[4]{*}{Traffic} & \multirow{4}[1]{*}{2} &ML-CW-U & 0.3750  & 0.2360  & 0.2773  & 0.2419  & 1.0970  & 0.9301  \\
          &       & $k$Fool & 0.4680  & 0.3485  & 0.3743  & 0.3448  & 1.0950  & 9.1090  \\
          &       & T$_{k}$ML-AP-U & 0.3830  & 0.2405  & 0.2850  & 0.2485  & 1.1060  & 0.7939  \\
          &       & $\textup{T}_k$MIA(Ours) & \textbf{0.2420} & \textbf{0.1530} & \textbf{0.1705} & \textbf{0.1487} & \textbf{1.1150} & \textbf{0.5297} \\
\cmidrule{2-9}          & \multirow{4}[1]{*}{3} &ML-CW-U & 0.4669  & 0.2455  & 0.3100  & 0.2576  & 1.0128  & 1.3553  \\
          &       & $k$Fool & 0.5395  & 0.3946  & 0.4326  & 0.3875  & 1.0779  & 9.5689  \\
          &       & T$_{k}$ML-AP-U & 0.4698  & 0.2469  & 0.3154  & 0.2622  & 1.0326  & 1.1458  \\
          &       & $\textup{T}_k$MIA(Ours) & \textbf{0.2500} & \textbf{0.1213} & \textbf{0.1337} & \textbf{0.1059} & \textbf{1.0849} & \textbf{0.5889} \\
\cmidrule{2-9}          & \multirow{4}[1]{*}{5} &ML-CW-U & 0.3843  & 0.2047  & 0.2812  & 0.2283  & 0.8235  & 1.6596  \\
          &       & $k$Fool & 0.4471  & 0.3867  & 0.4279  & 0.4279  & 1.0510  & 14.5481  \\
          &       & T$_{k}$ML-AP-U & 0.3608  & 0.2125  & 0.2944  & 0.2435  & 0.8902  & 1.4711  \\
          &       & $\textup{T}_k$MIA(Ours) & \textbf{0.1725} & \textbf{0.0816} & \textbf{0.0863} & \textbf{0.0612} & \textbf{1.0588} & \textbf{0.6231} \\
    \bottomrule
    \end{tabular}}%
  \label{tab:nus_global_rest_maxiter300}%
\end{table*}%

\subsection{Overall Performance under Scheme 2}

In the main paper, we only present the performance of different methods under Scheme 2 with the maximum iteration of 300 and $\delta = 2$ on NUS. In this part, Table \ref{tab:voc_random_maxiter100_pert2} to Table \ref{tab:nus_random_maxiter500_pert3} present the overall performance of different methods under various parameters on three datasets in the random selection scheme, where the maximum iteration ranges from \{100, 300, 500\} and $\delta = 2, 3$. 

From the results on VOC and COCO, we could more intuitively find the superiority of our method. But for the results on NUS, which are shown in Table \ref{tab:nus_random_maxiter100_pert2}-\ref{tab:nus_random_maxiter500_pert3}, the $\Delta l$ and APer values of our perturbation are not the best. We think there are two possible reasons. The first one is that since the model training on NUS can only achieve 0.828 mAP which is shown in Table \ref{tab:modelinfo}, its original classification effect is relatively poor. This means there exist more irrelevant labels in the top-$k$ set. To satisfy measure imperceptibility, the model needs to take more effort to push them outside the top-$k$ position; (2) as our perturbation is required to achieve both visual and measure imperceptibility, it has to make a trade-off between the perturbed number and perturbation size in some cases. We could see that $k$Fool usually presents a better average attack effectiveness while generating an obvious perturbation, and T$_k$ML-AP-U makes a smaller APer but overlook the impact on metrics. In contrast, our T$_k$MIA could achieve both properties and avoid these problems. Therefore, these results and analyses further validate the effectiveness of our perturbation and optimization framework.

\begin{figure*}[htbp]
  \centering
  \vspace{-0.2cm}
  \includegraphics[width=\linewidth]{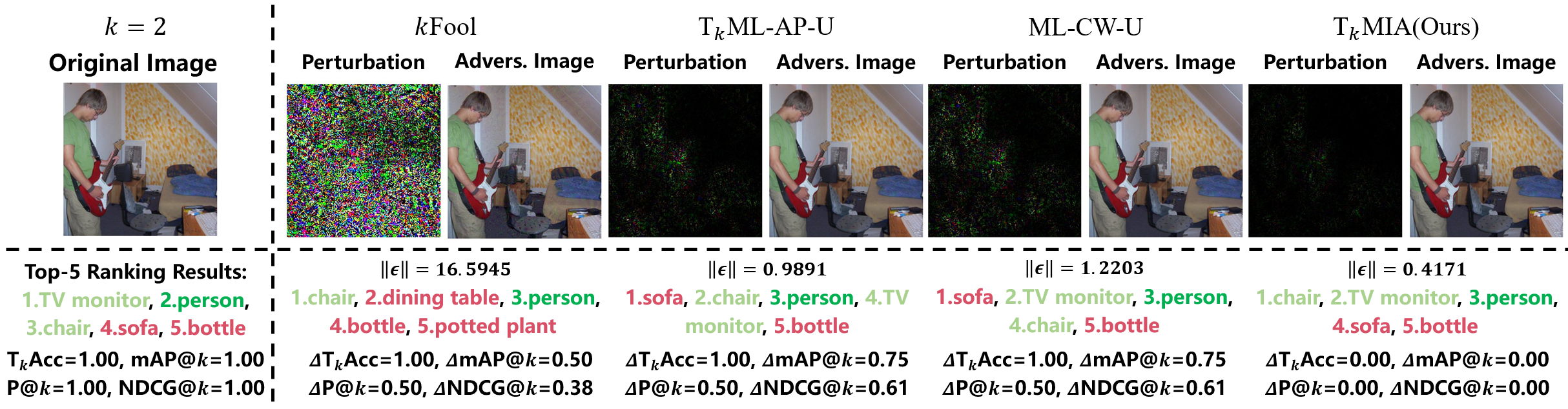}
\caption{The top-2 successful performance comparisons under \textit{global selection scheme} on VOC, where the maximum iteration is 300. All perturbation intensities are magnified by a factor of 100 to enhance contrast and visibility. The specified labels, relevant labels, and irrelevant labels are marked with \textcolor[rgb]{ 0,  .69,  .314}{dark green}, \textcolor[rgb]{ .573,  .816,  .314}{light green}, and \textcolor[rgb]{ .847,  .31,  .4}{red}, respectively.}
\label{fig:voc_global_example}
\end{figure*}

\begin{figure*}[htbp]
  \centering
  \vspace{-0.2cm}
  \includegraphics[width=\linewidth]{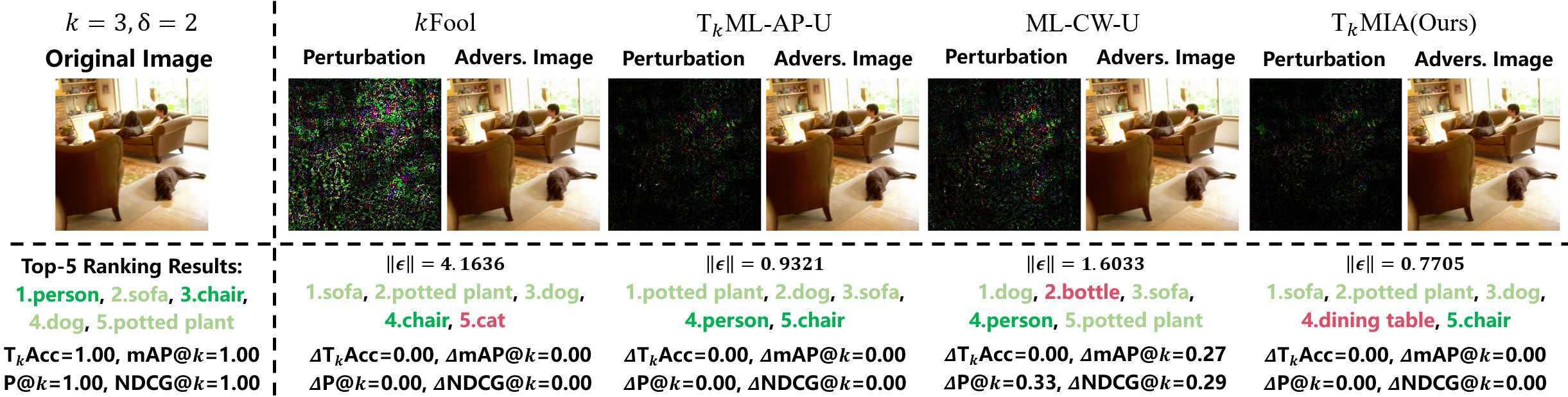}
\caption{The top-3 successful performance comparisons under \textit{random selection scheme} on VOC, where the maximum iteration is 300. All perturbation intensities are magnified by a factor of 100 to enhance contrast and visibility. The specified labels, relevant labels, and irrelevant labels are marked with \textcolor[rgb]{ 0,  .69,  .314}{dark green}, \textcolor[rgb]{ .573,  .816,  .314}{light green}, and \textcolor[rgb]{ .847,  .31,  .4}{red}, respectively.}
\label{fig:voc_random_example}
\end{figure*}

\begin{figure*}[t]
  \centering
  \vspace{-0.2cm}
  \includegraphics[width=\linewidth]{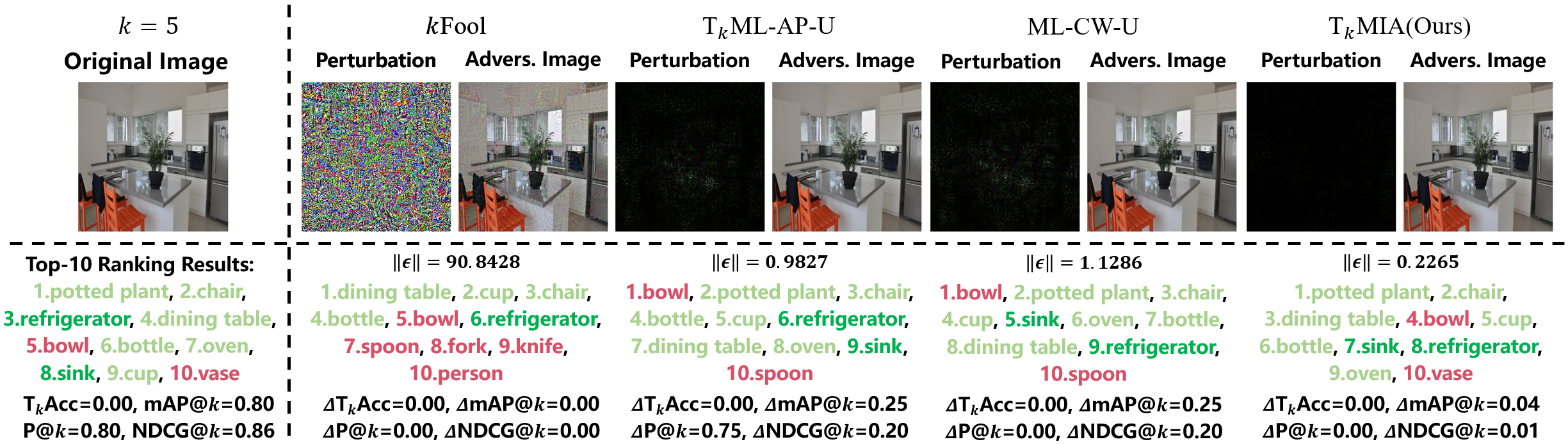}
\caption{The top-5 successful performance comparisons under \textit{global selection scheme} on COCO, where the maximum iteration is 300. All perturbation intensities are magnified by a factor of 100 to enhance contrast and visibility. The specified labels, relevant labels, and irrelevant labels are marked with \textcolor[rgb]{ 0,  .69,  .314}{dark green}, \textcolor[rgb]{ .573,  .816,  .314}{light green}, and \textcolor[rgb]{ .847,  .31,  .4}{red}, respectively.}
\label{fig:coco_global_example}
\end{figure*}
\vspace{-0.2cm}

\subsection{Visualization Results}

We set the maximum iteration as 300 and $\delta = 2$ to present the image results under Scheme 1 and Scheme 2 from Figure \ref{fig:voc_global_example} to Figure \ref{fig:coco_random_example}, where the first two sets of images are the performance on VOC and the others are on COCO. 

In Figure \ref{fig:voc_global_example}, we see that for our method, the metrics of the perturbed image remain unchanged. While other methods produce a clear performance degradation that causes a significant change in metrics. In Figure \ref{fig:voc_random_example}, it is not hard to find that even if all methods achieve the ideal results, our perturbation could craft a smaller perturbation to lower the ranking position of specified labels. Hence, the overall performance of our method demonstrates that the proposed perturbation is more effective than other attacks when pursuing both visual and measure imperceptibility. 

From the results on COCO, we could find that our perturbation manage to push the specified categories out of the top-$k$ region with slight efforts. Meanwhile, its changes in ranking-based metric values could be almost ignored, which achieves our perturbation effect. For $k$Fool, its perturbations always make the perturbed images visually perceptible though it could fool all metrics in some cases. Although T$_k$ML-AP-U performs well in some images, its overall effect is not better than our perturbation. This is consistent with our analysis shown in Sec.\ref{subsec:results}.

\begin{table*}[htbp]
  \centering
  \caption{The results of competitors and our method with the maximum iteration 300 under different $k$ values and \underline{globally} selected $\mathcal{S}$ on \underline{VOC}, where $\Delta$ refers to the difference between the original value and the perturbed value of corresponding metrics. $\downarrow$ means the smaller the value the better, and $\uparrow$ is the opposite. The best results under each set of parameters are bolded.}
  \renewcommand\arraystretch{0.9}
  \newcommand{\tabincell}[2]{\begin{tabular}{@{}#1@{}}#2\end{tabular}}
  \setlength{\tabcolsep}{3.5mm}{
    \begin{tabular}{c|cc|cccccc}
    \toprule
    Type & $k$ & Methods & $\Delta \textup{T}_k \textup{Acc} \downarrow$ & $\Delta \textup{P}@k \downarrow$ & $\Delta \textup{mAP}@k \downarrow$ & $\Delta \textup{NDCG}@k \downarrow$ & $\Delta l \uparrow$ & $\textup{APer} \downarrow$ \\
    \midrule
    \midrule
    \multirow{12}[4]{*}{Vehicles} & \multirow{4}[1]{*}{1} &ML-CW-U & 0.3452  & 0.3452  & 0.3452  & 0.3452  & 1.0000  & 1.0053  \\
          &       & $k$Fool & 0.1978  & 0.1978  & 0.1978  & 0.1978  & 0.9973  & 1.8162  \\
          &       & T$_{k}$ML-AP-U & 0.3489  & 0.3489  & 0.3489  & 0.3489  & 1.0000  & 0.8861  \\
          &       & $\textup{T}_k$MIA(Ours) & \textbf{0.1698} & \textbf{0.1698} & \textbf{0.1698} & \textbf{0.1698} & \textbf{1.0000} & \textbf{0.7644} \\
\cmidrule{2-9}          & \multirow{4}[1]{*}{2} &ML-CW-U & 0.4615  & 0.2308  & 0.3173  & 0.2569  & 1.0000  & 1.7331  \\
          &       & $k$Fool & 0.6923  & 0.3846  & 0.4519  & 0.3759  & 0.8846  & 23.1394  \\
          &       & T$_{k}$ML-AP-U & 0.5000  & 0.2500  & 0.3365  & 0.2717  & 1.0384  & 1.4915  \\
          &       & $\textup{T}_k$MIA(Ours) & \textbf{0.1538} & \textbf{0.0769} & \textbf{0.0769} & \textbf{0.0595} & \textbf{1.0384} & \textbf{0.9265} \\
\cmidrule{2-9}          & \multirow{4}[1]{*}{3} &ML-CW-U & 0.5000  & 0.2778  & 0.4352  & 0.3622  & 1.0000  & 1.3096  \\
          &       & $k$Fool & 0.5000  & 0.4444  & 0.4630  & 0.4218  & 1.0000  & 23.3179  \\
          &       & T$_{k}$ML-AP-U & 0.5000  & 0.2222  & 0.3425  & 0.2737  & 1.0000  & 1.1204  \\
          &       & $\textup{T}_k$MIA(Ours) & \textbf{0.1666} & \textbf{0.1111} & \textbf{0.1111} & \textbf{0.0884} & \textbf{1.0000} & \textbf{0.4735} \\
    \midrule
    \multirow{12}[4]{*}{Animals} & \multirow{4}[1]{*}{1} &ML-CW-U & 0.3333  & 0.3333  & 0.3333  & 0.3333  & 1.0000  & 0.7220  \\
          &       & $k$Fool & 0.2639  & 0.2639  & 0.2639  & 0.2639  & 1.0000  & \textbf{0.5135} \\
          &       & T$_{k}$ML-AP-U & 0.3287  & 0.3287  & 0.3287  & 0.3287  & 1.0000  & 0.6297  \\
          &       & $\textup{T}_k$MIA(Ours) & \textbf{0.1851} & \textbf{0.1851} & \textbf{0.1851} & \textbf{0.1851} & \textbf{1.0000} & 0.5758  \\
\cmidrule{2-9}          & \multirow{4}[1]{*}{2} &ML-CW-U & 0.5556  & 0.3111  & 0.4056  & 0.3413  & 1.0000  & 1.3934  \\
          &       & $k$Fool & 0.6000  & 0.4000  & 0.4278  & 0.3799  & 0.9111  & 4.3568  \\
          &       & T$_{k}$ML-AP-U & 0.5333  & 0.3222  & 0.4055  & 0.3498  & 1.0000  & 1.1133  \\
          &       & $\textup{T}_k$MIA(Ours) & \textbf{0.1555} & \textbf{0.0777} & \textbf{0.0777} & \textbf{0.0601} & \textbf{1.0000} & \textbf{0.6988} \\
\cmidrule{2-9}          & \multirow{4}[1]{*}{3} &ML-CW-U & 0.4286  & 0.1429  & 0.2619  & 0.2011  & 1.0000  & 1.1452  \\
          &       & $k$Fool & 0.4286  & 0.4286  & 0.4286  & 0.3791  & 1.0000  & 8.2682  \\
          &       & T$_{k}$ML-AP-U & 0.5714  & 0.2857  & 0.4206  & 0.3439  & 1.0000  & 1.1069  \\
          &       & $\textup{T}_k$MIA(Ours) & \textbf{0.1428} & \textbf{0.0952} & \textbf{0.1111} & \textbf{0.0845} & \textbf{1.0000} & \textbf{0.3470} \\
    \midrule
    \multirow{12}[4]{*}{Household} & \multirow{4}[1]{*}{1} &ML-CW-U & 0.2500  & 0.2500  & 0.2500  & 0.2500  & 0.9333  & 0.4442  \\
          &       & $k$Fool & 0.2750  & 0.2750  & 0.2750  & 0.2750  & 0.9500  & 1.3002  \\
          &       & T$_{k}$ML-AP-U & 0.2833  & 0.2833  & 0.2833  & 0.2833  & 1.0000  & 0.5046  \\
          &       & $\textup{T}_k$MIA(Ours) & \textbf{0.1583} & \textbf{0.1583} & \textbf{0.1583} & \textbf{0.1583} & \textbf{1.0000} & \textbf{0.4421} \\
\cmidrule{2-9}          & \multirow{4}[1]{*}{2} &ML-CW-U & 0.6071  & 0.3393  & 0.4375  & 0.3676  & 1.0000  & 1.6331  \\
          &       & $k$Fool & 0.5714  & 0.3571  & 0.3929  & 0.3410  & 0.4286  & 9.2461  \\
          &       & T$_{k}$ML-AP-U & 0.4642  & 0.2857  & 0.3571  & 0.3099  & 1.0000  & 1.2804  \\
          &       & $\textup{T}_k$MIA(Ours) & \textbf{0.2142} & \textbf{0.1250} & \textbf{0.1339} & \textbf{0.1128} & \textbf{1.0000} & \textbf{0.5842} \\
\cmidrule{2-9}          & \multirow{4}[1]{*}{3} &ML-CW-U & 1.0000  & 0.6667  & 0.8333  & 0.7039  & 1.0000  & 2.9472  \\
          &       & $k$Fool & 1.0000  & 0.6667  & 0.6667  & 0.5307  & 1.0000  & 6.6913  \\
          &       & T$_{k}$ML-AP-U & 1.0000  & 0.6666  & 0.6666  & 0.5307  & 1.0000  & 2.3044  \\
          &       & $\textup{T}_k$MIA(Ours) & \textbf{1.0000} & \textbf{0.3333} & \textbf{0.3333} & \textbf{0.2346} & \textbf{1.0000} & \textbf{0.2817} \\
    \midrule
    \multirow{12}[4]{*}{Person} & \multirow{4}[1]{*}{1} &ML-CW-U & 0.3288  & 0.3288  & 0.3288  & 0.3288  & 0.9955  & 1.0851  \\
          &       & $k$Fool & 0.2568  & 0.2568  & 0.2568  & 0.2568  & 1.0000  & \textbf{0.5595} \\
          &       & T$_{k}$ML-AP-U & 0.3063  & 0.3063  & 0.3063  & 0.3063  & 0.9966  & 0.9056  \\
          &       & $\textup{T}_k$MIA(Ours) & \textbf{0.1182} & \textbf{0.1182} & \textbf{0.1182} & \textbf{0.1182} & \textbf{1.0000} & 0.7834  \\
\cmidrule{2-9}          & \multirow{4}[1]{*}{2} &ML-CW-U & 0.5204  & 0.3056  & 0.4130  & 0.3542  & 0.9875  & 1.7190  \\
          &       & $k$Fool & 0.4013  & 0.2414  & 0.2712  & 0.2322  & 0.8307  & 5.6123  \\
          &       & T$_{k}$ML-AP-U & 0.4984  & 0.3009  & 0.3957  & 0.3420  & 1.0000  & 1.4489  \\
          &       & $\textup{T}_k$MIA(Ours) & \textbf{0.0721} & \textbf{0.0454} & \textbf{0.0595} & \textbf{0.0521} & \textbf{1.0000} & \textbf{1.1268} \\
\cmidrule{2-9}          & \multirow{4}[1]{*}{3} &ML-CW-U & 0.4576  & 0.2542  & 0.3964  & 0.3325  & 1.0000  & 1.8746  \\
          &       & $k$Fool & 0.3559  & 0.1808  & 0.2015  & 0.1557  & 0.8305  & 3.8721  \\
          &       & T$_{k}$ML-AP-U & 0.4406  & 0.2485  & 0.3926  & 0.3306  & 0.9491  & 1.6421  \\
          &       & $\textup{T}_k$MIA(Ours) & \textbf{-0.0338} & \textbf{0.0169} & \textbf{0.0348} & \textbf{0.0311} & \textbf{1.0000} & \textbf{1.2935} \\
    \bottomrule
    \end{tabular}}%
  \label{tab:voc_global_maxiter300}%
\end{table*}%

\begin{table*}[htbp]
  \centering
  \caption{The results of competitors and our method with the maximum iteration 100 and $\delta = 2$ on \underline{VOC} under different $k$ values and sizes of \underline{randomly} selected $\mathcal{S}$, where $\Delta$ refers to the difference between the original value and the perturbed value of corresponding metrics. $\downarrow$ means the smaller the value the better, and $\uparrow$ is the opposite. The best results under each set of parameters are bolded.}
  \renewcommand\arraystretch{0.9}
  \setlength{\tabcolsep}{4mm}{
    \begin{tabular}{c|cc|cccccc}
    \toprule
    $k$   & $\left | \mathcal{S} \right |$ & Methods & $\Delta \textup{T}_k \textup{Acc} \downarrow$ & $\Delta \textup{P}@k \downarrow$ & $\Delta \textup{mAP}@k \downarrow$ & $\Delta \textup{NDCG}@k \downarrow$ & $\Delta l \uparrow$ & $\textup{APer} \downarrow$ \\
    \midrule
    \midrule
    \multirow{4}[1]{*}{2} & \multirow{4}[1]{*}{2} & ML-CW-U & 0.5072  & 0.3043  & 0.4022  & 0.3470  & 1.3478  & 1.4079  \\
          &       & $k$Fool & 0.5507  & 0.3188  & 0.3478  & 0.2926  & 1.4492  & 12.2294  \\
          &       & T$_{k}$ML-AP-U & 0.5652  & 0.3913  & 0.4673  & 0.4208  & 1.5797  & \textbf{1.2688} \\
          &       & $\textup{T}_k$MIA(Ours) & \textbf{0.3043} & \textbf{0.2028} & \textbf{0.2246} & \textbf{0.1996} & \textbf{1.9855} & 1.4456  \\
    \midrule
    \multirow{4}[1]{*}{3} & \multirow{4}[1]{*}{2} & ML-CW-U & 0.1818  & 0.1212  & 0.2273  & 0.1976  & 1.4545  & 1.3738  \\
          &       & $k$Fool & 0.1818  & 0.1212  & 0.1565  & 0.1335  & 1.0000  & 25.7817  \\
          &       & T$_{k}$ML-AP-U & 1.4545  & 1.1405  & 0.0909  & 0.0909  & 0.1666  & \textbf{0.1493} \\
          &       & $\textup{T}_k$MIA(Ours) & \textbf{0.0000} & \textbf{0.0606} & \textbf{0.0707} & \textbf{0.0695} & \textbf{2.0000} & 1.2743  \\
    \bottomrule
    \end{tabular}}%
  \label{tab:voc_random_maxiter100_pert2}%
\end{table*}%

\begin{table*}[htbp]
  \centering
  \caption{The results of competitors and our method with the maximum iteration 300 and $\delta = 2$ on \underline{VOC} under different $k$ values and sizes of \underline{randomly} selected $\mathcal{S}$, where $\Delta$ refers to the difference between the original value and the perturbed value of corresponding metrics. $\downarrow$ means the smaller the value the better, and $\uparrow$ is the opposite. The best results under each set of parameters are bolded.}
  \renewcommand\arraystretch{0.9}
  \newcommand{\tabincell}[2]{\begin{tabular}{@{}#1@{}}#2\end{tabular}}
  \setlength{\tabcolsep}{4mm}{
    \begin{tabular}{c|cc|cccccc}
    \toprule
    $k$   & $\left | \mathcal{S} \right |$ & Methods & $\Delta \textup{T}_k \textup{Acc} \downarrow$ & $\Delta \textup{P}@k \downarrow$ & $\Delta \textup{mAP}@k \downarrow$ & $\Delta \textup{NDCG}@k \downarrow$ & $\Delta l \uparrow$ & $\textup{APer} \downarrow$ \\
    \midrule
    \midrule
    \multirow{4}[1]{*}{2} & \multirow{4}[1]{*}{2} & ML-CW-U & 0.7101  & 0.4638  & 0.5688  & 0.5031  & 1.9130  & 1.8886  \\
          &       & $k$Fool & 0.5507  & 0.3478  & 0.3768  & 0.3281  & 1.5072  & 17.3492  \\
          &       & T$_{k}$ML-AP-U & 0.6956  & 0.4782  & 0.5688  & 0.5110  & 1.9420  & 1.5410  \\
          &       & $\textup{T}_k$MIA(Ours) & \textbf{0.2898} & \textbf{0.1956} & \textbf{0.2246} & \textbf{0.2005} & \textbf{2.0000} & \textbf{1.4626} \\
    \midrule
    \multirow{4}[1]{*}{3} & \multirow{4}[1]{*}{2} & ML-CW-U & 0.3636  & 0.1818  & 0.3333  & 0.2773  & 1.8182  & 1.8550  \\
          &       & $k$Fool & 0.1818  & 0.1212  & 0.1414  & 0.1178  & 1.0909  & 27.6438  \\
          &       & T$_{k}$ML-AP-U & 0.3636  & 0.1818  & 0.3181  & 0.2615  & 2.0000  & 1.6427  \\
          &       & $\textup{T}_k$MIA(Ours) & \textbf{0.0000} & \textbf{0.0606} & \textbf{0.0707} & \textbf{0.0695} & \textbf{2.0000} & \textbf{1.2769} \\
    \bottomrule
    \end{tabular}}%
  \label{tab:voc_random_maxiter300_pert2}%
\end{table*}%

\begin{table*}[htbp]
  \centering
  \caption{The results of competitors and our method with the maximum iteration 100 and $\delta = 2$ on \underline{COCO} under different $k$ values and sizes of \underline{randomly} selected $\mathcal{S}$, where $\Delta$ refers to the difference between the original value and the perturbed value of corresponding metrics. $\downarrow$ means the smaller the value the better, and $\uparrow$ is the opposite. The best results under each set of parameters are bolded.}
  \renewcommand\arraystretch{0.9}
  \setlength{\tabcolsep}{4mm}{
    \begin{tabular}{c|cc|cccccc}
    \toprule
    $k$   & $\left | \mathcal{S} \right |$ & Methods & $\Delta \textup{T}_k \textup{Acc} \downarrow$ & $\Delta \textup{P}@k \downarrow$ & $\Delta \textup{mAP}@k \downarrow$ & $\Delta \textup{NDCG}@k \downarrow$ & $\Delta l \uparrow$ & $\textup{APer} \downarrow$ \\
    \midrule
    \midrule
    \multirow{8}[2]{*}{3} & \multirow{4}[1]{*}{2} &ML-CW-U & 0.6470  & 0.4900  & 0.6068  & 0.5478  & 1.8650  & 1.4768  \\
          &       & $k$Fool & 0.6190  & 0.3860  & 0.4428  & 0.3733  & 1.5900  & 10.9493  \\
          &       & T$_{k}$ML-AP-U & 0.6470  & 0.4810  & 0.6062  & 0.5443  & 1.9280  & 1.2629  \\
          &       & $\textup{T}_k$MIA(Ours) & \textbf{0.3700} & \textbf{0.2250} & \textbf{0.2652} & \textbf{0.2213} & \textbf{1.9970} & \textbf{1.0867} \\
\cmidrule{2-9}          & \multirow{4}[1]{*}{3} &ML-CW-U & 0.7300  & 0.3417  & 0.4647  & 0.3715  & 1.9050  & 1.4998  \\
          &       & $k$Fool & 0.6470  & 0.2746  & 0.3104  & 0.3104  & 1.9270  & 6.9378  \\
          &       & T$_{k}$ML-AP-U & 0.7330  & 0.3400  & 0.4703  & 0.3762  & 1.9780  & 1.2742  \\
          &       & $\textup{T}_k$MIA(Ours) & \textbf{0.5060} & \textbf{0.2166} & \textbf{0.2755} & \textbf{0.2114} & \textbf{2.0540} & \textbf{1.2499} \\
    \midrule
    \multirow{12}[4]{*}{5} & \multirow{4}[1]{*}{2} &ML-CW-U & 0.4620  & 0.4660  & 0.6133  & 0.5554  & 1.7510  & 1.8564  \\
          &       & $k$Fool & 0.4520  & 0.3630  & 0.4465  & 0.3655  & 1.5520  & 22.5433  \\
          &       & T$_{k}$ML-AP-U & 0.4610  & 0.4678  & 0.6231  & 0.5641  & 1.8540  & 1.6072  \\
          &       & $\textup{T}_k$MIA(Ours) & \textbf{0.2830} & \textbf{0.1828} & \textbf{0.2265} & \textbf{0.1739} & \textbf{1.9830} & \textbf{1.2318} \\
\cmidrule{2-9}          & \multirow{4}[1]{*}{3} &ML-CW-U & 0.5420  & 0.3104  & 0.4624  & 0.3788  & 1.8570  & 1.7330  \\
          &       & $k$Fool & 0.4980  & 0.2558  & 0.3229  & 0.2420  & 1.8240  & 13.9070  \\
          &       & T$_{k}$ML-AP-U & 0.5530  & 0.3086  & 0.4722  & 0.3868  & 1.9200  & 1.4759  \\
          &       & $\textup{T}_k$MIA(Ours) & \textbf{0.3480} & \textbf{0.1690} & \textbf{0.2301} & \textbf{0.1699} & \textbf{2.0550} & \textbf{1.2978} \\
\cmidrule{2-9}          & \multirow{4}[1]{*}{5} &ML-CW-U & 0.6311  & 0.1621  & 0.2189  & 0.1527  & 1.9417  & 1.4682  \\
          &       & $k$Fool & 0.5194  & 0.1271  & 0.1531  & 0.0986  & \textbf{2.0291}  & 5.9820  \\
          &       & T$_{k}$ML-AP-U & 0.5776  & 0.1446  & 0.1994  & 0.1386  & 1.9514  & 1.2364  \\
          &       & $\textup{T}_k$MIA(Ours) & \textbf{0.4660} & \textbf{0.1097} & \textbf{0.1328} & \textbf{0.0870} & 2.0048 & \textbf{1.1073} \\
    \midrule
    \multirow{12}[4]{*}{10} & \multirow{4}[1]{*}{2} &ML-CW-U & 0.1111  & 0.4270  & 0.5957  & 0.5494  & 1.6984  & 2.1748  \\
          &       & $k$Fool & 0.1111  & 0.3158  & 0.4231  & 0.3408  & 1.3174  & 63.3469  \\
          &       & T$_{k}$ML-AP-U & 0.4063  & 0.4063  & 0.5802  & 0.5359  & 1.7142  & 1.8900  \\
          &       & $\textup{T}_k$MIA(Ours) & \textbf{0.0793} & \textbf{0.1698} & \textbf{0.2141} & \textbf{0.1501} & \textbf{1.9682} & \textbf{1.2652} \\
\cmidrule{2-9}          & \multirow{4}[1]{*}{3} &ML-CW-U & 0.1905  & 0.3286  & 0.4994  & 0.4258  & 1.9524  & 2.0335  \\
          &       & $k$Fool & 0.1904  & 0.2428  & 0.3426  & 0.2379  & 1.8095  & 41.4890  \\
          &       & T$_{k}$ML-AP-U & 0.1904  & 0.2761  & 0.4655  & 0.3921  & 1.8571  & 1.6844  \\
          &       & $\textup{T}_k$MIA(Ours) & \textbf{0.1904} & \textbf{0.1857} & \textbf{0.2730} & \textbf{0.1908} & \textbf{2.0476} & \textbf{1.5473} \\
\cmidrule{2-9}          & \multirow{4}[1]{*}{5} &ML-CW-U & 1.0000  & 0.2000  & 0.2056  & 0.1315  & 2.0000  & 2.4546  \\
          &       & $k$Fool & 1.0000  & 0.1499  & 0.1661  & 0.1012  & 2.0000  & 15.6337  \\
          &       & T$_{k}$ML-AP-U & 1.0000  & 0.1499  & 0.2414  & 0.1737  & 2.0000  & 2.2539  \\
          &       & $\textup{T}_k$MIA(Ours) & \textbf{1.0000} & \textbf{0.0999} & \textbf{0.1289} & \textbf{0.0723} & \textbf{2.0000} & \textbf{1.8203} \\
    \bottomrule
    \end{tabular}}%
  \label{tab:coco_random_maxiter100_pert2}%
\end{table*}%

\begin{table*}[htbp]
  \centering
  \caption{The rest results of competitors and our method with the maximum iteration 300 and $\delta = 2$ on \underline{COCO} under different $k$ values and sizes of \underline{randomly} selected $\mathcal{S}$, where $\Delta$ refers to the difference between the original value and the perturbed value of corresponding metrics. $\downarrow$ means the smaller the value the better, and $\uparrow$ is the opposite. The best results under each set of parameters are bolded.}
  \renewcommand\arraystretch{0.9}
  \setlength{\tabcolsep}{4mm}{
    \begin{tabular}{c|cc|cccccc}
    \toprule
    $k$   & $\left | \mathcal{S} \right |$ & Methods & $\Delta \textup{T}_k \textup{Acc} \downarrow$ & $\Delta \textup{P}@k \downarrow$ & $\Delta \textup{mAP}@k \downarrow$ & $\Delta \textup{NDCG}@k \downarrow$ & $\Delta l \uparrow$ & $\textup{APer} \downarrow$ \\
    \midrule
    \midrule
    \multirow{8}[2]{*}{3} & \multirow{4}[1]{*}{2} &ML-CW-U & 0.6850  & 0.5367  & 0.6568  & 0.5965  & 2.0000  & 1.6477  \\
          &       & $k$Fool & 0.6260  & 0.3943  & 0.4527  & 0.3833  & 1.6360  & 18.2784  \\
          &       & T$_{k}$ML-AP-U & 0.6640  & 0.4933  & 0.6239  & 0.5599  & 2.0000  & 1.3426  \\
          &       & $\textup{T}_k$MIA(Ours) & \textbf{0.3740} & \textbf{0.2240} & \textbf{0.2642} & \textbf{0.2203} & \textbf{2.0000} & \textbf{1.0935} \\
\cmidrule{2-9}          & \multirow{4}[1]{*}{3} &ML-CW-U & 0.7820  & 0.3693  & 0.5004  & 0.4012  & 2.0250  & 1.6598  \\
          &       & $k$Fool & 0.6580  & 0.2813  & 0.3183  & 0.2376  & 1.9470  & 10.0718  \\
          &       & T$_{k}$ML-AP-U & 0.7540  & 0.3533  & 0.4887  & 0.3919  & 2.0260  & 1.3381  \\
          &       & $\textup{T}_k$MIA(Ours) & \textbf{0.4860} & \textbf{0.2087} & \textbf{0.2684} & \textbf{0.2063} & \textbf{2.0460} & \textbf{1.2561} \\
    \midrule
    \multirow{12}[4]{*}{10} & \multirow{4}[1]{*}{2} &ML-CW-U & 0.1111  & 0.4635  & 0.6450  & 0.5938  & 2.0000  & 2.3947  \\
          &       & $k$Fool & 0.1111  & 0.2952  & 0.4132  & 0.3331  & 1.3333  & 61.5793  \\
          &       & T$_{k}$ML-AP-U & \textbf{0.1111} & 0.4555  & 0.6351  & 0.5859  & 2.0000  & 2.0634  \\
          &       & $\textup{T}_k$MIA(Ours) & 0.3450  & \textbf{0.1704} & \textbf{0.2319} & \textbf{0.1716} & \textbf{2.0710} & \textbf{1.3283} \\
\cmidrule{2-9}          & \multirow{4}[1]{*}{3} &ML-CW-U & 0.1905  & 0.3381  & 0.5166  & 0.4420  & 2.0000  & 2.1017  \\
          &       & $k$Fool & \textbf{0.1904}  & 0.2476  & 0.3374  & 0.2406  & 1.8095  & 38.4887  \\
          &       & T$_{k}$ML-AP-U & 0.1904  & 0.3095  & 0.4889  & 0.4174  & 2.0000  & 1.8292  \\
          &       & $\textup{T}_k$MIA(Ours) & 0.4514 & \textbf{0.1087} & \textbf{0.1308} & \textbf{0.0862} & \textbf{2.0145} & \textbf{1.1295} \\
\cmidrule{2-9}          & \multirow{4}[1]{*}{5} &ML-CW-U & \textbf{0.5000} & 0.1000  & \textbf{0.1163} & \textbf{0.0698} & 2.0000  & \textbf{1.6352} \\
          &       & $k$Fool & 0.5000  & 0.0999  & 0.1211  & 0.0714  & 2.0000  & 17.5303  \\
          &       & T$_{k}$ML-AP-U & 0.5000  & 0.0999  & 0.1914  & 0.1418  & 2.0000  & 2.2226  \\
          &       & $\textup{T}_k$MIA(Ours) & 1.0000  & \textbf{0.0999} & 0.1472  & 0.0805  & \textbf{2.0000} & 1.6473  \\
    \bottomrule
    \end{tabular}}%
  \label{tab:coco_random_rest_maxiter300_pert2}%
\end{table*}%

\begin{table*}[htbp]
  \centering
  \caption{The results of competitors and our method with the maximum iteration 500 and $\delta = 3$ on \underline{COCO} under different $k$ values and sizes of \underline{randomly} selected $\mathcal{S}$, where $\Delta$ refers to the difference between the original value and the perturbed value of corresponding metrics. $\downarrow$ means the smaller the value the better, and $\uparrow$ is the opposite. The best results under each set of parameters are bolded.}
  \renewcommand\arraystretch{0.9}
  \setlength{\tabcolsep}{4mm}{
    \begin{tabular}{c|cc|cccccc}
    \toprule
    $k$   & $\left | \mathcal{S} \right |$ & Methods & $\Delta \textup{T}_k \textup{Acc} \downarrow$ & $\Delta \textup{P}@k \downarrow$ & $\Delta \textup{mAP}@k \downarrow$ & $\Delta \textup{NDCG}@k \downarrow$ & $\Delta l \uparrow$ & $\textup{APer} \downarrow$ \\
    \midrule
    \midrule
    \multirow{4}[1]{*}{3} & \multirow{4}[1]{*}{3} &ML-CW-U & 0.9740  & 0.7177  & 0.8457  & 0.7674  & 3.0000  & 2.0312  \\
          &       & $k$Fool & 0.8300  & 0.4526  & 0.5194  & 0.4236  & 2.2590  & 32.7336  \\
          &       & T$_{k}$ML-AP-U & 0.9630  & 0.6996  & 0.8305  & 0.7513  & 3.0000  & 1.6714  \\
          &       & $\textup{T}_k$MIA(Ours) & \textbf{0.5890} & \textbf{0.3166} & \textbf{0.3773} & \textbf{0.3108} & \textbf{3.0000} & \textbf{1.5890} \\
    \midrule
    \multirow{8}[2]{*}{5} & \multirow{4}[1]{*}{3} &ML-CW-U & 0.6390  & 0.6286  & 0.7825  & 0.7180  & 3.0000  & 2.2634  \\
          &       & $k$Fool & 0.5690  & 0.4020  & 0.5021  & 0.4124  & 2.1600  & 100.0775  \\
          &       & T$_{k}$ML-AP-U & 0.6370  & 0.6128  & 0.7758  & 0.7078  & 3.0000  & 1.8995  \\
          &       & $\textup{T}_k$MIA(Ours) & \textbf{0.4280} & \textbf{0.2480} & \textbf{0.3110} & \textbf{0.2395} & \textbf{3.0000} & \textbf{1.5833} \\
\cmidrule{2-9}          & \multirow{4}[1]{*}{5} &ML-CW-U & 0.9272  & 0.3592  & 0.5382  & 0.4132  & 3.0243  & 2.0656  \\
          &       & $k$Fool & 0.7087  & 0.2359  & 0.2947  & \textbf{0.2044}  & 2.7475  & 35.2941  \\
          &       & T$_{k}$ML-AP-U & 0.9077  & 0.3262  & 0.5086  & 0.3872  & 3.0145  & 1.7110  \\
          &       & $\textup{T}_k$MIA(Ours) & \textbf{0.6553} & \textbf{0.2106} & \textbf{0.2909} & 0.2053 & \textbf{3.0388} & \textbf{1.6672} \\
    \midrule
    \multirow{8}[2]{*}{10} & \multirow{4}[1]{*}{3} &ML-CW-U & 0.1905  & 0.5905  & 0.7404  & 0.6973  & 3.0000  & 2.4183  \\
          &       & $k$Fool & 0.1904  & 0.3476  & 0.4753  & 0.3706  & 1.8571  & 101.8643  \\
          &       & T$_{k}$ML-AP-U & 0.1904  & 0.5571  & 0.7413  & 0.6858  & 3.0000  & 2.1472  \\
          &       & $\textup{T}_k$MIA(Ours) & \textbf{0.1904} & \textbf{0.2380} & \textbf{0.3179} & \textbf{0.2297} & \textbf{3.0000} & \textbf{1.7015} \\
\cmidrule{2-9}          & \multirow{4}[1]{*}{5} &ML-CW-U & 1.0000  & 0.4000  & 0.6626  & 0.5137  & 3.0000  & 2.9090  \\
          &       & $k$Fool & 1.0000  & \textbf{0.2500}  & \textbf{0.2865}  & \textbf{0.1769}  & 2.0000  & 10.2294  \\
          &       & T$_{k}$ML-AP-U & 1.0000  & 0.3999  & 0.5386  & 0.4225  & 3.0000  & 2.5256  \\
          &       & $\textup{T}_k$MIA(Ours) & \textbf{1.0000} & 0.3000 & 0.4126 & 0.2507 & \textbf{3.0000} & \textbf{2.3050} \\
    \bottomrule
    \end{tabular}}%
  \label{tab:coco_random_maxiter500_pert3}%
\end{table*}%

\begin{table*}[htbp]
  \centering
  \caption{The results of competitors and our method with the maximum iteration 100 and $\delta = 2$ on \underline{NUS} under different $k$ values and sizes of \underline{randomly} selected $\mathcal{S}$, where $\Delta$ refers to the difference between the original value and the perturbed value of corresponding metrics. $\downarrow$ means the smaller the value the better, and $\uparrow$ is the opposite. The best results under each set of parameters are bolded.}
  \renewcommand\arraystretch{0.9}
  \setlength{\tabcolsep}{4mm}{
    \begin{tabular}{c|cc|cccccc}
    \toprule
    $k$   & $\left | \mathcal{S} \right |$ & Methods & $\Delta \textup{T}_k \textup{Acc} \downarrow$ & $\Delta \textup{P}@k \downarrow$ & $\Delta \textup{mAP}@k \downarrow$ & $\Delta \textup{NDCG}@k \downarrow$ & $\Delta l \uparrow$ & $\textup{APer} \downarrow$ \\
    \midrule
    \midrule
    \multirow{4}[1]{*}{2} & \multirow{4}[1]{*}{2} &ML-CW-U & 0.4690  & 0.3240  & 0.3710  & 0.3337  & 1.2820  & 1.3689  \\
          &       & $k$Fool & 0.7420  & 0.5160  & 0.5595  & 0.5042  & \textbf{1.8820} & 19.0300  \\
          &       & T$_{k}$ML-AP-U & 0.5270  & 0.3670  & 0.4235  & 0.3819  & 1.3780  & \textbf{1.1840} \\
          &       & $\textup{T}_k$MIA(Ours) & \textbf{0.3740} & \textbf{0.2325} & \textbf{0.2600} & \textbf{0.2253} & 1.5990  & 1.2962  \\
    \midrule
    \multirow{8}[2]{*}{3} & \multirow{4}[1]{*}{2} &ML-CW-U & 0.2980  & 0.1963  & 0.2546  & 0.2240  & 0.9200  & 1.5029  \\
          &       & $k$Fool & 0.6710  & 0.5436  & 0.5909  & 0.5443  & \textbf{1.9110} & 44.2922  \\
          &       & T$_{k}$ML-AP-U & 0.3570  & 0.2466  & 0.3167  & 0.2811  & 1.0300  & \textbf{1.3011} \\
          &       & $\textup{T}_k$MIA(Ours) & \textbf{0.2560} & \textbf{0.1536} & \textbf{0.1775} & \textbf{0.1473} & 1.3580  & 1.3524  \\
\cmidrule{2-9}          & \multirow{4}[1]{*}{3} &ML-CW-U & \textbf{0.3420}  & \textbf{0.1433}  & 0.1851  & 0.1436  & 1.0320  & 1.4344  \\
          &       & $k$Fool & 0.6340  & 0.2740  & 0.3226  & 0.2462  & \textbf{2.0050} & 12.0781  \\
          &       & T$_{k}$ML-AP-U & 0.4050  & 0.1743  & 0.2347  & 0.1844  & 1.1380  & \textbf{1.2434} \\
          &       & $\textup{T}_k$MIA(Ours) & 0.3770 & 0.1486 & \textbf{0.1730} & \textbf{0.1280} & 1.3880  & 1.3605  \\
    \midrule
    \multirow{12}[4]{*}{5} & \multirow{4}[1]{*}{2} &ML-CW-U & 0.2670  & 0.1332  & 0.1819  & 0.1477  & 0.4920  & 1.6755  \\
          &       & $k$Fool & 0.6020  & 0.5782  & 0.6469  & 0.5876  & \textbf{1.8700} & 85.8662  \\
          &       & T$_{k}$ML-AP-U & 0.2820  & 0.1678  & 0.2263  & 0.1918  & 0.6020  & 1.4976  \\
          &       & $\textup{T}_k$MIA(Ours) & \textbf{0.2480} & \textbf{0.1196} & \textbf{0.1397} & \textbf{0.1008} & 1.0620  & \textbf{1.4137} \\
\cmidrule{2-9}          & \multirow{4}[1]{*}{3} &ML-CW-U & 0.2522  & 0.1009  & 0.1334  & 0.1022  & 0.6959  & 1.6779  \\
          &       & $k$Fool & 0.6171  & 0.3506  & 0.4237  & 0.3335  & \textbf{2.0125} & 35.2714  \\
          &       & T$_{k}$ML-AP-U & 0.2593  & 0.1202  & 0.1674  & 0.1330  & 0.7781  & \textbf{1.4853} \\
          &       & $\textup{T}_k$MIA(Ours) & \textbf{0.2432} & \textbf{0.0994} & \textbf{0.1194} & \textbf{0.0846} & 1.0930  & 1.5154  \\
\cmidrule{2-9}          & \multirow{4}[1]{*}{5} &ML-CW-U & 0.3750  & 0.0833  & 0.1124  & 0.0716  & 1.2083  & 1.2708  \\
          &       & $k$Fool & 0.5416  & 0.1333  & 0.1572  & 0.1048  & \textbf{2.1666} & 3.2104  \\
          &       & T$_{k}$ML-AP-U & \textbf{0.2916} & \textbf{0.0750} & 0.1001  & 0.0693  & 1.3750  & 1.1406  \\
          &       & $\textup{T}_k$MIA(Ours) & 0.3333  & 0.0833  & \textbf{0.0936} & \textbf{0.0609} & 1.5416  & \textbf{1.0140} \\
    \bottomrule
    \end{tabular}}%
  \label{tab:nus_random_maxiter100_pert2}%
\end{table*}%

\begin{table*}[htbp]
  \centering
  \caption{The rest results of competitors and our method with the maximum iteration 300 and $\delta = 2$ on \underline{NUS} under different $k$ values and sizes of \underline{randomly} selected $\mathcal{S}$, where $\Delta$ refers to the difference between the original value and the perturbed value of corresponding metrics. $\downarrow$ means the smaller the value the better, and $\uparrow$ is the opposite. The best results under each set of parameters are bolded.}
  \renewcommand\arraystretch{0.9}
  \newcommand{\tabincell}[2]{\begin{tabular}{@{}#1@{}}#2\end{tabular}}
  \setlength{\tabcolsep}{4mm}{
    \begin{tabular}{c|cc|cccccc}
    \toprule
    $k$   & $\left | \mathcal{S} \right |$ & Methods & $\Delta \textup{T}_k \textup{Acc} \downarrow$ & $\Delta \textup{P}@k \downarrow$ & $\Delta \textup{mAP}@k \downarrow$ & $\Delta \textup{NDCG}@k \downarrow$ & $\Delta l \uparrow$ & $\textup{APer} \downarrow$ \\
    \midrule
    \midrule
    \multirow{4}[1]{*}{2} & \multirow{4}[1]{*}{2} &ML-CW-U & 0.7630  & 0.5495  & 0.6230  & 0.5677  & 1.8840  & 2.0010  \\
          &       & $k$Fool & 0.7530  & 0.5285  & 0.5737  & 0.5186  & 1.9030  & 28.4019  \\
          &       & T$_{k}$ML-AP-U & 0.7930  & 0.5720  & 0.6550  & 0.5971  & \textbf{1.9250} & \textbf{1.6244} \\
          &       & $\textup{T}_k$MIA(Ours) & \textbf{0.5220} & \textbf{0.3475} & \textbf{0.3837} & \textbf{0.3408} & 1.9170  & 1.6469  \\
    \midrule
    \multirow{8}[2]{*}{3} & \multirow{4}[1]{*}{2} &ML-CW-U & 0.6190  & 0.4503  & 0.5606  & 0.5017  & 1.7210 & 2.2771  \\
          &       & $k$Fool & 0.6680  & 0.5396  & 0.5857  & 0.5385  & \textbf{1.9140}  & 49.0967  \\
          &       & T$_{k}$ML-AP-U & 0.6640  & 0.5053  & 0.6261  & 0.5637  & 1.8390  & 1.8981  \\
          &       & $\textup{T}_k$MIA(Ours) & \textbf{0.4110} & \textbf{0.2673} & \textbf{0.3030} & \textbf{0.2577} & 1.8160  & \textbf{1.8075} \\
\cmidrule{2-9}          & \multirow{4}[1]{*}{3} &ML-CW-U & 0.7090  & 0.3383  & 0.4390  & 0.3504  & 1.7770  & 2.2154  \\
          &       & $k$Fool & 0.6360  & 0.2733  & 0.3235  & 0.2467  & \textbf{2.0040} & 13.1977  \\
          &       & T$_{k}$ML-AP-U & 0.7770  & 0.3790  & 0.5008  & 0.4029  & 1.8780  & \textbf{1.8377} \\
          &       & $\textup{T}_k$MIA(Ours) & \textbf{0.6010} & \textbf{0.2730} & \textbf{0.3158} & \textbf{0.2412} & 1.8880  & 1.9296  \\
    \bottomrule
    \end{tabular}}%
  \label{tab:nus_random_rest_maxiter300_pert2}%
\end{table*}%

\begin{table*}[htbp]
  \centering
  \caption{The results of competitors and our method with the maximum iteration 500 and $\delta = 3$ on \underline{NUS} under different $k$ values and sizes of \underline{randomly} selected $\mathcal{S}$, where $\Delta$ refers to the difference between the original value and the perturbed value of corresponding metrics. $\downarrow$ means the smaller the value the better, and $\uparrow$ is the opposite. The best results under each set of parameters are bolded.}
  \renewcommand\arraystretch{0.9}
  \setlength{\tabcolsep}{4mm}{
    \begin{tabular}{c|c|c|cccccc}
    \toprule
    $k$   & $\left | \mathcal{S} \right |$ & Methods & $\Delta \textup{T}_k \textup{Acc} \downarrow$ & $\Delta \textup{P}@k \downarrow$ & $\Delta \textup{mAP}@k \downarrow$ & $\Delta \textup{NDCG}@k \downarrow$ & $\Delta l \uparrow$ & $\textup{APer} \downarrow$ \\
    \midrule
    \midrule
    \multirow{4}[1]{*}{3} & \multirow{4}[1]{*}{3} &ML-CW-U & 0.9340  & 0.6833  & 0.8049  & 0.7308  & 2.8030  & 2.8624  \\
          &       & $k$Fool & 0.9630  & 0.7353  & 0.8002  & 0.7337  & 2.8630  & 83.3383  \\
          &       & T$_{k}$ML-AP-U & 0.9680  & 0.7450  & 0.8606  & 0.7917  & \textbf{2.9390} & \textbf{2.3481} \\
          &       & $\textup{T}_k$MIA(Ours) & \textbf{0.7680} & \textbf{0.4893} & \textbf{0.5452} & \textbf{0.4729} & 2.8600  & 2.5979  \\
    \midrule
    \multirow{8}[2]{*}{5} & \multirow{4}[1]{*}{3} &ML-CW-U & 0.6565  & 0.5574  & 0.7041  & 0.6318  & 2.5671  & 3.2198  \\
          &       & $k$Fool & 0.6905  & 0.6991  & 0.7644  & 0.7095  & 2.8246  & 162.6897  \\
          &       & T$_{k}$ML-AP-U & 0.6762  & 0.6314  & 0.7772  & 0.7104  & \textbf{2.8389} & \textbf{2.6935} \\
          &       & $\textup{T}_k$MIA(Ours) & \textbf{0.5402} & \textbf{0.3667} & \textbf{0.4086} & \textbf{0.3237} & 2.6583  & 2.7339  \\
\cmidrule{2-9}          & \multirow{4}[1]{*}{5} &ML-CW-U & 0.8750  & 0.3417  & 0.4689  & 0.3557  & 2.5833  & 2.9803  \\
          &       & $k$Fool & \textbf{0.7500}  & 0.2666  & 0.3309  & 0.2337  & \textbf{3.0000} & 15.4970  \\
          &       & T$_{k}$ML-AP-U & 0.9166  & 0.3666  & 0.5490  & 0.4223  & 2.9166  & \textbf{2.4811} \\
          &       & $\textup{T}_k$MIA(Ours) & 0.7916 & \textbf{0.2749} & \textbf{0.3258} & \textbf{0.2214} & 2.9580  & 2.5540  \\
    \bottomrule
    \end{tabular}}%
  \label{tab:nus_random_maxiter500_pert3}%
\end{table*}%

\begin{figure*}[t]
  \centering
  \includegraphics[width=\linewidth]{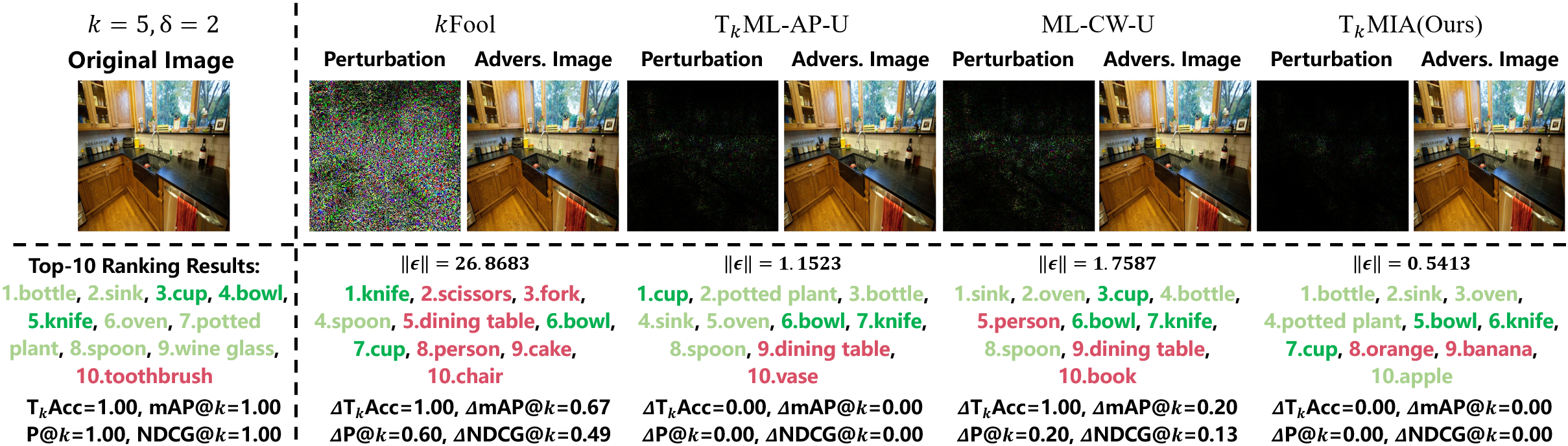}
\caption{The top-5 successful performance comparisons under \textit{random selection scheme} on COCO, where the maximum iteration is 300. All perturbation intensities are magnified by a factor of 100 to enhance contrast and visibility. The specified labels, relevant labels, and irrelevant labels are marked with \textcolor[rgb]{ 0,  .69,  .314}{dark green}, \textcolor[rgb]{ .573,  .816,  .314}{light green}, and \textcolor[rgb]{ .847,  .31,  .4}{red}, respectively.}
\label{fig:coco_random_example}
\end{figure*}

\end{document}